\documentclass[11pt]{amsart}
\usepackage{amscd}
\usepackage[arrow,matrix]{xy}
\input xypic
\usepackage{graphicx, tikz}
\usepackage{hyperref}
\usepackage{soul}
\usepackage{comment}
\usepackage{amsmath}
\usepackage{amsmath, latexsym, amssymb}
\numberwithin{equation}{section}
\theoremstyle{plain}

\newtheorem{lemma}{Lemma}[section]
\newtheorem{proposition}[lemma]{Proposition}
\newtheorem{theorem}[lemma]{Theorem}

\theoremstyle{definition}
\newtheorem{definition}[lemma]{Definition}
\newtheorem*{definition*}{Definition}
\newtheorem{remark}[lemma]{Remark}
\newtheorem{example}[lemma]{Example}

\usepackage{color}
\usepackage{cancel}

\definecolor{brown}{RGB}{150,100,0}

\definecolor{purple}{RGB}{150,0,100}
\definecolor{grey}{RGB}{128,128,128}

\newcommand{\R}{{\mathbb R}}

\newcommand{\E}{{\mathbb E}}

\newcommand{\N}{{\mathbb N}}

\newcommand{\Aa}{{\mathcal A}}
\newcommand{\Bb}{{\mathcal B}}

\newcommand{\Ee}{{\mathcal E}}

\newcommand{\Hh}{{\mathcal H}}

\newcommand{\Ll}{{\mathcal L}}    
\newcommand{\Mm}{{\mathcal M}}    

\newcommand{\Pp}{{\mathcal P}}

\newcommand{\Ss}{{\mathcal S}}

\newcommand{\Xx}{{\mathcal X}}
\newcommand{\Yy}{{\mathcal Y}}
\newcommand{\Zz}{{\mathcal Z}}

\newcommand{\Om}{{\Omega}}

\newcommand{\eps}{{\varepsilon}}

\newcommand{\vol}{{\rm vol}}

\newcommand{\Meas}{\mathbf{Meas}\,}

\newcommand{\Probm}{\mathbf{Probm}\,}

\renewcommand{\b}{{\mathfrak b}}

\newcommand{\pb}{{\mathbf p}}

\newcommand{\grad}{{\nabla}}

\newcommand{\la}{\langle}
\newcommand{\ra}{\rangle}

\newcommand{\INTO}{\hookrightarrow}

\newcommand{\be}{\begin{equation}}
	\newcommand{\bel}[1]{\begin{equation}\label{#1}}
		\newcommand{\qe}{\end{equation}}
	\newcommand{\ee}{\end{equation}}
\newcommand{\eeq}{\end{equation}}
\newcommand{\ba}{\begin{eqnarray}}
\newcommand{\ea}{\end{eqnarray}}

\newcommand{\Id}{ {\rm Id}}

\mathchardef\mhyp="2D

\newcommand{\Ho}{{\mathrm{H}\kern 0.3pt}}
\newcommand{\Zl}{{\mathrm{Z}\kern 0.2pt}}
\newcommand{\Bd}{{\mathrm{B}\kern 0.2pt}}

\newcommand{\Hcal}{\mathcal{H}}
\newcommand{\mapto}{\rightarrow}

\newcommand{\equivalent}{\ensuremath{\Longleftrightarrow}}
\newcommand{\trace}{\mathrm{trace}}

\newcommand{\Xbf}{\mathbf{X}}

\newcommand{\Xcal}{\mathcal{X}}

\newcommand{\Lcal}{\mathcal{L}}

\newcommand{\Deltabf}{\mathbf{\Delta}}

\newcommand{\Mrm}{\mathrm{M}}
\newcommand{\logdet}{\mathrm{logdet}}

\newcommand{\approach}{\ensuremath{\rightarrow}}

\newcommand{\Sym}{\mathrm{Sym}}
\newcommand{\myone}{\mathbf{1}}
\newcommand{\Nbb}{\mathbb{N}}
\newcommand{\ai}{\mathrm{ai}}
\newcommand{\bw}{\mathrm{bw}}
\newcommand{\logE}{\mathrm{logE}}
\newcommand{\myast}{\circledast}
\newcommand{\HS}{\mathrm{HS}}
\newcommand{\logHS}{\mathrm{logHS}}
\newcommand{\eHS}{\mathrm{HS_X}}

\begin{document}

\title[Categorical  and geometric methods   in statistical learning]{Categorical  and geometric methods  in statistical, manifold,  and machine learning}
\author{H\^ong V\^an L\^e}
	\address{Institute of Mathematics of the Czech Academy of Sciences, Zitna 25, 11567 Praha 1,  $\&$   Charles University, Faculty of Mathematics and Physics
		Ke Karlovu 3, 121 16 Praha 2, Czechia  }
		 \email{hvle@math.cas.cz}
\author{H\`a Quang Minh}
\address{RIKEN Center for Advanced Intelligence Project,
	Tokyo, Japan}
	\email{minh.haquang@riken.jp}
\author{Frederic Protin}
\address{Torus Action SAS, 3 Avenue Didier Daurat,
	Toulouse, 31400, France    $\&$  Institut de Mathematiques de Toulouse, Universit\'e Toulouse 3, 18 Route de Narbonne,Toulouse, 31400  France}
\email{fredprotin@yahoo.fr}	
\author{Wilderich  Tuschmann}
\address{Faculty of Mathematics, Karlsruhe  Institute  of Technology, Englerstr. 2, Karlsruhe, D-76131, Germany}
\email{tuschmann@kit.edu}

%
%
\maketitle
\begin{abstract}We present and  discuss  applications of the  category of probabilistic morphisms, initially developed in \cite{Le2023}, as well as some  geometric methods to  several  classes of problems in statistical, machine  and manifold  learning  which shall  be, along with many other topics, considered  in depth in the forthcoming book  \cite{LMPT2024}.
\end{abstract}

\section{Introduction}\label{sec:intro}
Let us start with  a   general concept  of  learning  and  machine  learning.
{\it Learning}   is a process  of gaining  new knowledge in the sense of obtaining new  correlations between  features of  observables by the examination  of  empirical data associated with a given  {\it finite  set} of  observables. When these characterizations  can be tested and validated through the examination of new associated data and  their  accuracy, and expressive and predictive power does improve by feeding such data, we speak of {\it successful learning}. 
Nowadays,  {\it machine learning} refers to any learning process in this sense which can be implemented
into and performed by a computing device. 
Finally, the theory or science of {\it Machine Learning} as a whole denotes the study, creation and application
of machine learning processes and techniques.

The mathematical theory for learning from data, using  probability theory and  mathematical statistics,  is called    statistical learning theory.  In  statistical learning theory, 
to   learn   from data, we  need  to perform the following   steps.

\underline{Step 1}.  Construct    a mathematical  model of learning, using probability theory and  mathematical statistics  to  model incomplete   information of the observed   data.   More  precisely,   given  a sample  space  $\Xx$,   we need  to  construct a    hypothesis  space $\Hh$  of  possible decisions  and     
model  our incomplete  information  of     data   $S_n \in \Xx^n$,  our uncertainty of   a  correct choice  of  a   decision  $h \in \Hh$, given    data $S_n$, by  using   probability  theory  and  mathematical statistics.

\underline{Step 2}.  Find  a learning algorithm  $A$, i.e.  a  map  $A: \cup_{n =1}^\infty  \Xx^n \to  \Hh$.

\underline{Step 3}.  Estimate  the  error  of  a learning  algorithm $A (S_n) \in  \Hh$,    to make sure  that the  learning   algorithm  is successful.

In Step 1 we consider  different types of learning problems, in particular types of training data, i.e., the  data available to the learner before making a decision/prediction. A  main type of statistical learning problems is  supervised learning where
training data are labeled, see Definition \ref{def:gen}, Example \ref{ex:gensuper}, which  we shall consider in   greater detail in  Section \ref{sec:probgen}, using the categorical concept  of Markov  kernels.

There  are  two  main (frequentist  and Bayesian) approaches    to Step 1, which    yield  different probabilistic modelings of    training data and their  stochastic relations and hence  to Step 2, see  \cite{Vapnik2000}, \cite{GM2007}, \cite{HTF2008}, \cite{Theodoridis2015}, \cite{MRT2018}.  In both  approaches of probabilistic modeling,   the concept  of  a Markov kernel  is  of central importance.
In our  paper we   discuss  applications  of  probabilistic  morphisms, which are categorical   realizations of    Markov kernels, to  statistical learning theory.   The category  of probabilistic  morphisms   was  proposed   independently by  Lawvere  in 1962 \cite{Lawvere1962}, motivated by  discrete stochastic processes,  and  by   Chentsov in 1972  \cite{Chentsov1972} as {\it the category of statistical    decisions} \cite[\S 5, p. 65]{Chentsov1972},  which is a natural continuation of    Chentsov' work  in 1965 \cite{Chentsov1965}.  
Methods  of probabilistic morphisms were developed  for   generative models  of supervised learning  by L\^e in \cite{Le2023}, \cite{Le2025} and   applied  by   us  to problems of   statistical learning  and machine  learning  in \cite{LMPT2024}.  We also  present  geometric  methods  to several  problems in statistical learning theory,   which   are complementary to    the categorical approach  in   our  book \cite{LMPT2024}.

Our  paper  is organized as follows.  In Section \ref{sec:probgen}, we first recall  the notion   of the category of probabilistic  morphisms. Basic examples  of probabilistic morphisms are   measurable   mappings, Markov kernels and regular  conditional probability measures that are of  central importance in  mathematical statistics  and  statistical  learning  theory. Using the concept of probabilistic morphisms  we present the notion of a generative model of supervised learning. Using  L\^e's characterization  of regular conditional  probability measures, based on the concept of a graph of  a probabilistic morphism (Theorem \ref{thm:marginal}),  we    give  an example  of  a correct  loss function of a  generative model  of supervised  learning (Examples \ref{ex:gensuper}, \ref{ex:losscatkernel}). In  the last part  of this section we outline  the  proof of  L\^e's Theorem on  the existence 
of  over-parameterized  supervised  learnable  models  (Proposition \ref{prop:learnable}). The learnability of this  model is proved  using a  version of  Vapnik-Stephanyuk's method of solving  stochastic  ill-posed  problems (Proposition \ref{thm:61Le2023}).  In Section \ref{sec:geomkernel}, we discuss the problem of generalizing the Gaussian kernel defined
on Euclidean spaces to metric spaces, in particular to Riemannian manifolds. After discussing the difficulties of such a generalization to general Riemannian manifolds, we present positive definite kernels defined using the Log-Euclidean metric on the manifold of symmetric positive definite (SPD) matrices and its infinite-dimensional generalization, the Log-Hilbert-Schmidt metric on the set of positive definite unitized Hilbert-Schmidt operators on a Hilbert space. We then discuss a special setting for the Log-Hilbert-Schmidt metric, namely,
 that of reproducing kernel Hilbert space (RKHS) covariance operators, where all quantities of interest admit closed form expressions via finite-dimensional Gram matrices and thus can be applied in practical applications.
In  Section \ref{sec:manifoldlearning},  we discuss  several  important topics  of manifold learning, which lies at the intersection of geometry  and   statistical learning.  In the final section, we summarize the main results of this paper  and discuss  some related problems.

\section{Probabilistic  morphisms and  generative models of supervised learning}\label{sec:probgen}
\subsection{Notation and conventions}\label{subs:notation}

- For a measurable space $\Xx$,  we  denote  by $\Sigma_\Xx$  the $\sigma$-algebra of $\Xx$,  and by $\Ss(\Yy), \Mm (\Yy), \Pp (\Yy)$ the  spaces  of  all signed finite  measures,   (nonnegative) finite  measures,  probability measures, respectively,   on $\Xx$.  If $\Xx$ is a  topological space,   we always  consider   the Borel   $\sigma$-algebra   $\Bb (\Xx)$,  unless otherwise stated.

- For  a nonnegative  measure $\mu$ on  a measurable space $\Xx$,  
we denote  by $\Ll ^{p} (\Xx, \mu)$, or just $\Ll ^{p}  (\mu)$,  the  class  of  all $\mu$-measurable functions $f$
such that $|  f ^p|$ is  a $\mu$-integrable  function.
Denote by $L ^p (\mu)$  the quotient  space    of $\Ll ^p (\mu)$  with  respect to the  equivalence   relation  $ f \stackrel{\mu}{\sim} g$ if $ f =g $ $\mu$-a.e..  Following tradition,   we    use     the expression ``a function $f$ in  $L ^p (\mu)$", where one
should say ``a function $f$ in  $\Ll  ^p (\mu)$” when this does not lead to misunderstanding.

-  For  any    measurable  space $\Xx$,  we  denote by $\Sigma _w   $  the smallest $\sigma$-algebra  on $\Pp (\Xx)$ such that  for  any $A \in \Sigma _\Xx$  the function $I_{ 1_A}: \Pp (\Xx) \to \R, \mu \mapsto   \int 1_A  d\mu $ is  measurable.  Here $1_A$ is the  characteristic  function  of $A$.  In our paper,  we always  consider  $\Pp (\Xx)$  as  a measurable   space  with    the   $\sigma$-algebra  $\Sigma_w$, unless  otherwise stated.

-  A Markov  kernel $T: \Xx \times  \Sigma_\Yy \to [0,1]$    is  uniquely defined   by    the map  $ \overline T:  \Xx \to  \Pp (\Yy)$  such that  $\overline T (x) (A) = T (x, A)$ for all $x\in \Xx, A \in \Sigma _\Yy$.   We  shall   also use notations $T ( A|x):= T (x, A) $ and $\overline T(A|x) : = \overline T (x)(A)$.  

- A    {\it probabilistic  morphism}  $T: \Xx \leadsto \Yy$  is an  arrow assigned to a  measurable mapping, denoted  by $\overline T$  from $ \Xx$ to  $\Pp (\Yy)$.  We say  that  $T$ is generated  by $\overline  T$.  For  a measurable mapping $T: \Xx \to \Pp (\Yy)$ we  denote  by $\underline T: \Xx \leadsto  \Yy$ the generated  probabilistic morphism.

-  For   probabilistic  morphisms  $T_{\Yy|\Xx} : \Xx \leadsto \Yy$ and  $T_{\Zz|\Yy}: \Yy \leadsto  \Zz$  their    composition is the  probabilistic  morphism
$$  T_{\Zz|\Xx} : = T_{\Zz |\Yy} \circ  T _{\Yy|\Xx}: \Xx \leadsto \Zz,  $$
$$	(T_{\Zz|\Yy}\circ T_{\Yy|\Xx}) (x, C): = \int_\Yy T_{\Zz|\Yy} (y ,C) T_{\Yy| \Xx} (dy|x) $$
for	$x\in \Xx$   and $C \in \Sigma_\Zz$.  This   corresponds to   the composition of the associated  Markov  kernels, and hence  the composition is associative.

- We  denote  by $\Meas(\Xx, \Yy)$  the set  of  all measurable mappings from   a measurable  space  $\Xx$ to a measurable space $\Yy$, and by $\Probm (\Xx, \Yy)$  the set  of all  probabilistic  morphisms    from $\Xx$ to $\Yy$.
We denote  by $\Yy ^\Xx$  the set  of all mappings  from  a set $\Xx$ to  a set  $\Yy$.
For  any set $\Xx$ we denote by $\Id_\Xx$ the  identity map on $\Xx$. For a   product   space  $\Xx \times \Yy$ we denote  by $\Pi_\Xx$ the canonical projection   to the    factor  $\Xx$.   For $\mu \in \Pp  (\Xx\times \Yy)$ denote by $\mu_\Xx$  the marginal   probability measure  $(\Pi_\Xx)_* \mu\in \Pp (\Xx)$. 

Important  examples  of  probabilistic morphisms  are   measurable  mappings  $\kappa: \Xx \to \Yy$ since    the  Dirac map $\delta:\Xx \to \Pp (\Xx), x\mapsto \delta_x,$ is measurable \cite{Giry1982},  and hence  we  regard $\kappa$ as a  probabilistic morphism   which is generated  by the measurable mapping $\overline \kappa: =\delta \circ \kappa: \Xx \to \Pp (\Yy)$.  We  shall denote    measurable mappings by straight  arrows  and  probabilistic morphisms  by curved  arrows.  Hence,  $\Meas (\Xx, \Yy) $  can be regarded as a  subset of $\Probm(\Xx, \Yy)$.  Other important  examples  of  probabilistic morphisms  in statistical  learning  are regular  conditional   probability measures,  whose   definition we    recall now.  Let $\mu$  be   a finite measure on $(\Xx \times \Yy, \Sigma_\Xx \otimes \Sigma_\Yy)$.  A   {\it product regular  conditional  probability measure} \index{product  regular conditional measure} for $\mu$ with respect  to the  projection $\Pi_\Xx: \Xx \times \Yy \to \Xx$ is a  Markov  kernel\index{Markov kernel} $\mu_{\Yy|\Xx}: \Xx \times \Sigma_\Yy \to [0,1]$ such that
\begin{equation}\label{eq:regpr}
	\mu (A \times B) = \int_A  \mu_{\Yy|\Xx} (x, B)\, d (\Pi_\Xx)_* \mu  (x)
\end{equation}
for  any $A \in \Sigma_\Xx$ and $B \in \Sigma_\Yy$. We shall    simply call $\mu_{\Yy|\Xx}$ a {\it regular  conditional  probability measure} for $\mu$  and identify $\mu_{\Yy|\Xx}$ with the  generating measurable  map $\overline{\mu_{\Yy|\Xx}}: \Xx \to \Pp(\Yy), x\mapsto \mu_{\Yy|\Xx} (\cdot |x)$.

- We denote by $\mathbf {Meas}$ the category of measurable  spaces, whose objects  are measurable  spaces   and  morphisms  are measurable mappings.

\subsection{The category $\Probm$ and characterizations  of regular  conditional  probability  measures} \label{subs:probm}

\subsubsection{Category of  probabilistic morphisms}\label{ssub:probm}
We denote by $\mathbf{Probm}$  the category  whose objects  are measurable  spaces   and  morphisms  are of  Markov kernels  $ T:  \Xx \times  \Sigma_\Yy \to [0,1]$ as    morphisms  from $\Xx$ to $\Yy$.  The identity   Markov kernel $\underline{\delta \circ  \Id _\Xx}: \Xx \times \Sigma_\Xx \to [0,1]$is defined as follows: $(x, A)\mapsto \delta _x (A)$.   For $\kappa \in \Meas(\Xx, \Yy)$ we  also use the shorthand  notation
\begin{equation}\label{eq:delta}
	\overline{\kappa}: = \delta \circ  \kappa.
	\end{equation}
Identifying  $\Meas (\Xx, \Yy)$  with  a subset   in  $\Probm (\Xx, \Yy) =\Meas(\Xx, \Pp (\Yy))$ via the composition with the Dirac map $\delta$:  $\Meas (\Xx, \Yy)
\ni \kappa \mapsto \delta \circ \kappa \in \Probm (\Xx, \Yy)$, we  regard  $\Meas$ as a subcategory of $\Probm$.

The category 
$\Probm$ admits  a faithful  functor $S$ to the category  $\mathbf {Ban}$  of Banach  spaces   whose  morphisms  are bounded  linear mappings  of operator norm less than or equal to one  as follows. For any $T \in \Probm (\Xx, \Yy)$, $\mu \in \Ss (\Xx)$  and $B \in \Sigma_\Yy$ we set
\begin{equation}\label{eq:Mhomomorphism}
	S(T)_*\mu  (B) = \int _\Xx  \overline T(B|x) \, d\mu (x).
\end{equation}

\begin{proposition}\label{prop:chentsovbanach} \cite[Lemmas 5.9, 5.10]{Chentsov1972}  Let  $S$  assign  to each  measurable space  $\Xx$ the   Banach  space $\Ss(\Xx)$  of  finite  signed measures    on $\Xx$ endowed  with the  total variation norm  $\| \cdot \|_{TV}$ and  to every  Markov kernel  $T_{\Yy|\Xx}$ the Markov homomorphism  $(T_{\Yy|\Xx})_*$.
	Then  $S$ is a faithful  functor from  the category $\mathbf {Probm}$ to the category $\mathbf {Ban}$.
\end{proposition}

\begin{remark}\label{rem:functors} (1) It is known that the restriction $M_*(T)$  of $S_*(T)$ to  $\Mm(\Xx)$  and the  restriction $P_*(T)$  of $S_*(T)$  to  $\Pp (\Xx)$ maps  $\Mm(\Xx)$ to  $\Mm(\Yy)$ and $\Pp (\Xx)$ to
	$\Pp(\Yy)$, respectively  \cite[Lemma 5.9, p. 72]{Chentsov1972}.  We   shall use   the shorthand notation $T_*$  for $S_* (T)$, $M_* (T)$ and $P_* (T)$, if no misunderstanding occurs.
	
	(2) Let $T \in \Probm(\Xx, \Yy)$. Then  $T_*: \Pp(\Xx) \to \Pp(\Yy)$ is a measurable   mapping \cite[Theorem 1]{Giry1982}.
	
	(3) Let $T \in \Probm(\Xx, \Yy)$  and $\nu \ll \mu \in \Ss (\Xx)$. Then $T_*(\nu) \ll T_*(\mu)$ by  a result due  to 
	Ay-Jost-L\^e-Schwachh\"ofer \cite[Remark  5.4, p. 255]{AJLS2017}, which generalizes Morse-Sacksteder's  result \cite[Proposition 5.1]{MS1966}, see also \cite[Theorem 2 (2)]{JLT2021}  for an alternative  proof.
\end{remark}

Using  the functor $S$,   we shall     characterize   probabilistic   regular conditional probability   measures.
Given  two  measurable  mappings
$\overline T_i: \Xx \to \Pp(\Yy_i)$,  where i = 1,2,  let us consider the map $$\overline{T_1 \cdot T_2}: \Xx \to \Pp (\Yy_1 \times \Yy_2), \: x\mapsto   \mathfrak m (\overline   T_1(x), \overline T_2(x)),$$
where the multiplication $\mathfrak m$ is defined  as follows:
$$\mathfrak m (\mu_1, \mu_2) = \mu_1 \otimes \mu_2.$$
It is  easy to see  that $\mathfrak m$ is a measurable mapping.
The map $\overline{T_1 \cdot T_2}: \Xx \to \Pp (\Yy_1 \times \Yy_2)$  is
a measurable mapping  since  it is  the  composition of  two measurable  mappings $(\overline T_1,\overline T_2): \Xx  \to  \Pp (\Yy_1) \times \Pp (\Yy_2)$ and  $\mathfrak m: \Pp (\Yy_1) \times \Pp (\Yy_2) \to  \Pp (\Yy_1 \times \Yy_2)$.

\begin{definition}\label{def:graph}  (1)  Given  two  probabilistic morphisms
	$T_i: \Xx \leadsto \Yy_i$  for i = 1,2,     {\it the join  of  $T_1$ and $T_2$ }  is  the  probabilistic morphism  $T_1 \cdot T_2: \Xx \leadsto \Yy_1 \times \Yy_2$   whose  generating   mapping  is $\overline{T_1 \cdot T_2}: \Xx \to \Pp (\Yy_1 \times \Yy_2)$  given by
	$$\overline{T_1 \cdot T_2} (x): = \mathfrak m (\overline   T_1(x), \overline T_2(x)).$$
	
	(2)  Given  a        probabilistic  morphism  $T: \Xx \leadsto \Yy$  we denote  the join  of  $\Id_\Xx$ with  $T$ by
	$\Gamma  _T: \Xx \leadsto  \Xx \times  \Yy$  and call  it  the   {\it graph  of $T$}. 
\end{definition}

\begin{remark}\label{rem:joint} 
	(1) If  $\kappa: \Xx \to \Yy$  is  a measurable  mapping, regarded as a probabilistic morphism, then
	the graph $\Gamma_\kappa: \Xx \leadsto \Xx \times \Yy$  is a  measurable mapping  defined  by $x \mapsto  (x, \kappa (x))$  for $x\in \Xx$, since  $\overline \Id_\Xx = \delta \circ  \Id_\Xx$, and  therefore $\mathfrak m  (\overline\Id_\Xx (x),\overline \kappa (x) )   =  \delta  _x \otimes \delta _{\kappa (x)} = \delta \circ \Gamma _{\kappa} (x)$.
	
	
	(2) For historical  notes on the  concept  of   the  join  of   two  probabilistic  morphisms, we refer  to \cite[Remark 2.17 (1)]{Le2023}.
\end{remark}

	

	
	

\begin{definition}[Almost surely equality]\label{def:asep} \cite{Le2023}, \cite{Fritz2019}  Let $\mu \in \Pp(\Xx)$.   Two  measurable mappings $T, T':\Xx\to  \Pp(\Yy)$  will be called  {\it equal   $\mu$-a.e.} (with the shorthand notation $ T =  T' $    $\mu$-a.e.), if 
	for  any $B \in \Sigma_\Yy$
	$$\mu\{x\in \Xx: T(x)(B) \not =  T'(x)(B)\} = 0.$$
\end{definition}

\subsubsection{Characterizations of regular conditional probability measures}\label{subs:charreg}

\begin{theorem}[Characterization of regular conditional probability measures]\label{thm:marginal}
	
	\
	
	(i) A measurable mapping  $\overline T : \Xx \to \Pp (\Yy)$ is a regular conditional  probability measure for $\mu\in \Pp(\Xx \times \Yy)$   with  respect  to the    projection  $\Pi_\Xx$ if  and only if
	\begin{equation}\label{eq:graph}
		(\Gamma_T)_* \mu_\Xx =\mu.
	\end{equation}
	
	(ii)  If  $ \overline  T,\,  \overline T':\Xx \to \Pp (\Yy)$ are  regular conditional  probability measures for $\mu\in \Pp(\Xx \times \Yy)$   with  respect  to the    projection  $\Pi_\Xx$,  then  $  \overline  T =  \overline T'$ $\mu_\Xx$-a.e. Conversely,  if  $T$ is a  regular conditional  probability measure for $\mu\in \Pp(\Xx \times \Yy)$   with  respect  to the    projection  $\Pi_\Xx$  and  $  \overline  T =  \overline T'$ $\mu_\Xx$-a.e.,  then $\overline T':\Xx \to \Pp (\Yy)$ is a regular conditional  probability measure for $\mu\in \Pp(\Xx \times \Yy)$   with  respect  to the    projection  $\Pi_\Xx$.
\end{theorem}

Theorem \ref{thm:marginal} follows  easily  from the    definition of a regular conditional  probability measure (Equation \ref{eq:regpr}), see  also   \cite[Theorem 2.22]{Le2023}  for  a  slight  generalization.

\begin{example}[Measurement error]\label{ex:merror}   Assume   that  $x\in (\Xx, \mu_\Xx)$, $\mu _\Xx \in \Pp(\Xx)$, and   $y\in  \R^n$ are   related         by the following   equation
	\begin{equation}\label{eq:merror}
		y =  f (x) + \eps 
	\end{equation}
	where  $f \in \R^\Xx$  and  $\eps \in  (\R^n, \mu_\eps)$.  One   regards   $\eps$ being generated  by  a measurable mapping  $g_\eps:   (\Om, \mu_\Om) \to \R^n$, where  $(\Om, \mu_\Om)$ is a latent source space, $\mu_\Om \in \pb (\Om)$ and $\mu_\eps = (g_\eps)_* \mu_\Om$. One   interprets the  assumption  that the   noise $\eps$ is  independent  of $x\in \Xx$,  and        the    measurement error  equation  \eqref{eq:merror}  as  the equation for  the  conditional  probability $\mu_{\R^n |\Xx} (x)$ of  $y$  given   $x$,  which    satisfies
	\begin{equation}\label{eq:noisec}\overline{\mu_{\R^n |\Xx}} (x) =  \delta_{f ( x)}*  \mu _\eps \in \Pp (\R^n).
	\end{equation}
	Here  $ \mu_1 * \mu_2=  A_* (\mu_1 \otimes \mu_2)$ is the convolution  for   $\mu_1, \mu_2 \in \Pp (\R^n)$,  noting that $ A: \R^n \times \R^n \to \R^n, (x, y) \mapsto (x+y)$ is a measurable map.     Since  for   any $\mu_0 \in \Pp (\R^n)$ one  easily  shows that  the embedding 
	$I_{\mu_0}: \Pp (\R^n)\to \Pp (\R^n \otimes \R^n), \mu \mapsto \mu\otimes \mu_0,$  is a measurable mapping, we conclude that  the map $\overline{\mu_{\R^n|\Xx}}: \Xx \to \Pp (\Yy)$  defined  in \eqref{eq:noisec}  is a measurable map,  if  $f \in \Meas(\Xx, \R^n)$. In this  case $(x, y)$   is  distributed  jointly  by     the   probability measure  $\mu \in \Pp (\Xx \times \R^n)$  that  is defined  uniquely  by   its marginal  measure $\mu_\Xx = (\Pi_\Xx)_* \mu$ and  its  regular  conditional 
	probability  measure that is generated by the map $\overline{\mu_{\R^n|\Xx}}: \Xx \to \Pp (\R^n)$  defined in \eqref{eq:noisec}, i.e. $\mu = (\Gamma_{\mu_{\R^n|\Xx}})_* \mu_\Xx$.   If the  mean  $m (\mu_\eps) : = \int _{\R^n} y \, d\mu_\eps (y)$ of  $\mu_\eps$ is zero,  then   one    has
\begin{equation}\label{eq:regreg}
	f (x) =  r_\mu (x): = \int_{\R^n} y \, d \mu_{\R^n|\Xx} (y|x). 
\end{equation}
Formula  \eqref{eq:regreg}  explains  the  importance  and popularity  of the     problem of  estimating    the  regression function $r_\mu$, where  $\mu \in \Pp (\Xx \times \R^n)$ is  (unknown)
probability  measure,  in supervised  learning,  see  e.g. \cite[\S 1.4, p. 26]{Vapnik1998},  \cite{CS2001}. 
\end{example}

\subsection{Generative models  of supervised  learning}\label{subs:gens}

\subsubsection{Generative models  of supervised  learning and correct loss functions}\label{subs:gencorrect}
In   supervised learning,    given a data set  of labeled   items  $S_n = \{(x_1, y_1), \ldots, 
(x_n, y_n)\}\in  (\Xx \times \Yy)^n$,   the  aim  of a learner
is to find  a best approximation  $f_{S_n}$   of  the stochastic  relation between  $x\in \Xx$ and its labels  $y \in \Yy$, formalized as  a version of the conditional  probability measure  $\mu_{\Yy|\Xx}: \Xx \to \Pp (\Yy)$   expressing  the probability  of   the  label $y\in \Yy$  given an input  $x\in \Xx$.
In this   subsection,  we consider  the   setting  of frequentist  supervised  learning, where  $(x_i, y_i)$ are assumed  to  be  i.i.d., i.e.,
for any   $n$, $S_n \in  ((\Xx \times \Yy)^n, \mu^n)$ for  some unknown      $\mu \in \Pp (\Xx\times \Yy)$. Furthermore, we  assume that  $\mu$  admits a regular conditional  probability measure  $\mu_{\Yy|\Xx}$  with  respect  to the projection $\Pi_\Xx: \Xx \times \Yy \to \Xx$. It is known that  if  $\Yy$ is a Polish space,    and $\Xx$ is a measurable space,  then  for any   $\mu \in \Pp (\Xx \times \Yy)$ there exists  a regular   conditional  measure $\mu_{\Yy|\Xx}$   with respect  to $\Pi_\Xx: \Xx \times \Yy \to \Xx$,   see  e.g.  \cite[Corollary 10.4.15, p. 366, vol. 2]{Bogachev2007}, or  \cite[Theorem 3.1]{LFR2004}  for a more general result.

The  concept of a  best approximation requires   a specification of a hypothesis $\Hh$  of possible  predictors as  well   a  the notion  of a {\it correct} loss function  that measures  the  deviation  of a  possible predictor from a regular  conditional  probability  measure $\mu_{\Yy|\Xx}$.  The loss function concept   in statistical analysis was  introduced by Wald in his work on  statistical  decision  theory \cite{Wald1950} and has been intensively discussed  in  statistical learning  theory \cite{Vapnik1998}, \cite{Steinwart:SVM2008}.   Now  we   perform   Step  1 of   modeling      a  supervised  learning  problem under the above assumptions  in the  following  definition.

\begin{definition}\label{def:gen}  (cf. \cite[Definitions 3.1, 3.3]{Le2023})  {\it A generative model of supervised learning} is  given by a   quintuple  $(\Xx, \Yy, \Hh, R, \Pp_{\Xx \times \Yy})$,  where $\Xx$   and $\Yy$ are  measurable spaces,  $\Hh$  is a  family  of measurable  mappings $h: \Xx \to \Pp (\Yy)$, $\Pp_{\Xx \times \Yy}\subset \Pp (\Xx \times \Yy)$  contains  all possible   probability measures  governing  the  distributions of  labeled  pairs $(x, y)$,  and   $R : \Hh \times \Pp_{\Xx  \times \Yy} \to \R \cup \{ +\infty\}$ is  a risk/loss  function  such that $\inf_{h \in \Hh} R_\mu (h)\not = \pm \infty$  for any $\mu \in \Pp_{\Xx\times \Yy}$. If
$R (h, \mu) = \E_\mu  (L  (\cdot, \cdot, h))$ where  $L: \Xx  \times \Yy \times \Hh \to \R \cup \{ +\infty \}$ is an  instantaneous loss function such that  $L (\cdot, \cdot, h ) \in \Meas (\Xx \times \Yy, \R)$ for any $h \in \R$,  then we shall write   the model as  $(\Xx, \Yy, \Hh, R^L, \Pp_{\Xx \times \Yy})$. Note  that $\E_\mu f$ is well-defined, iff $h  \in L^1 (\mu)$, otherwise  we let  $\E_\mu (h ): = +\infty$. If $\Hh \subset \Meas (\Xx, \Yy) \subset \Meas (\Xx, \Pp(\Yy))$ we shall say that $(\Xx, \Yy, \Hh, R, \Pp_{\Xx \times \Yy})$ is a {\it  discriminative} model of supervised  learning.  If  $R$  can be  extended  to a loss function, also denoted  by $R:\Hh \times (\Pp_{\Xx  \times \Yy} \cup \Pp_{emp} (\Xx \times \Yy) )\to \R \cup \{ +\infty\}$, then $R$ is called  {\it empirically  definable}.
	
	 A  loss  function
	$ R: \Hh \times  \Pp_{\Xx \times \Yy} \to \R  \cup \{ + \infty\}$ will  be  called  {\it $\Pp_{\Xx\times \Yy}$-correct},
	if there exists a  set $\widetilde \Hh \subset \mathbf{Meas}  (\Xx, \Pp (\Yy))$ such  that  the following  three conditions   hold:
	
	(1) $\Hh \subset \widetilde  \Hh$.
	
	(2)  For any $\mu \in \Pp_{\Xx \times \Yy}$ there exists
	$h \in \widetilde \Hh$ such  that $h$ is  a regular  conditional measure  for $\mu$  relative to the  projection $\Pi_\Xx$.
	
	(3) $R$ is  the restriction of a  loss function $\widetilde  R: \widetilde \Hh \times \big (\Pp_{\Xx \times \Yy}\cup\Pp_{emp} (\Xx \times \Yy)\big )
	\to \R  \cup \{ + \infty\}$   such that  for  any  $\mu \in \Pp_{\Xx \times \Yy}$    
	$$\arg \min _{h \in \widetilde \Hh} \widetilde R (h, \mu) = \{ h \in \widetilde \Hh|\, h \text { is a regular  conditional probability measure  for $\mu$}\}.$$
	
	A  loss  function
	$ R: \Hh \times \big (\Pp_{\Xx \times \Yy}\big ) \to \R  \cup \{ + \infty\}$ will  be  called {\it correct}, if  $R$  is   the restriction of   a  $\Pp (\Xx \times \Yy)$-correct  loss function  $\widetilde R: \Hh	\times  \Pp (\Xx \times \Yy) \to \R  \cup \{ + \infty\}$.	
\end{definition}


\begin{example}\label{ex:gensuper}   Let us consider  the following generative model  of  classification problems $ (\Xx, \Yy, \Hh, R^L, \Pp_{\Xx \times \Yy})$, a particular case of supervised  learning  where  $\Yy = \{0, 1\}$, and $L$ shall  be defined  below  in \eqref{eq:instql}. Since $\Yy = \{0, 1\}$,   one can identify  $\Pp (\Yy)$  with the  interval  $[0,1]$, namely, 
	we  associate  a  probability  measure  $\pb \in  \Pp (\Yy)$ with the value  $\pb  (\{1\}) \in [0,1]$.   Since  $\Yy$ is finite,    the  strong topology on $\Pp (\Yy)$  coincides  with the  weak*-topology $\tau_w$, and therefore   $\Pp(\Yy)$  is  a measurable  space  isomorphic  to  $([0,1], \Bb([0,1])$.   Thus
	$$\Meas(\Xx, \Pp(\Yy))=\Meas(\Xx,[0,1]).$$ 	
 Recall that $\mu_{\Yy|\Xx}: \Xx \to \Pp (\Yy)$  denotes a regular  conditional probability measure  for  $\mu\in \Pp(\Xx \times \Yy9)$ and
	$$r_\mu(x)=\int_\Yy y\, d \mu_{\Yy|\Xx}(y| x) \in [0,1].$$
	Since  $r_\mu (x) = \mu_{\Yy|\Xx} (\{1\}|x)$,
	$\mu_{\Yy|\Xx}$
	and $r_\mu$ define each other. Now assume  that 
	$$\Hh   \subset \Meas   (\Xx, [0,1])  =  \Probm (\Xx, \Yy),$$
	where, recall that, in the  last  equality,  for $h \in \Meas (\Xx, [0,1])$,       we    identify $h(x)$   with the  measure  $\bar h (x) \in \Pp (\Yy)$ defined  by  $\bar h (x) (\{ 1\})  = h (x)$ and $\bar h (x) (\{ 0\}) = 1- h(x)$.  Clearly, $\Meas(\Xx, [0,1]) \subset  \Ll^2  (\Xx,  (\Pi_\Xx)_*\mu )$  for any 
	$\mu \in \Pp (\Xx \times \Yy)$. 
It is known that $r_\mu $ is the minimizer   of  the     expected loss  function
	$$R^L_\mu: \Ll^2 (\Xx,(\Pi_\Xx)_*\mu )\to   \R_{\ge 0},\:   h \mapsto   \E_\mu   L  ( \cdot, \cdot, h)$$
	where  $L$ is the  instantaneous   quadratic  function:
	\begin{equation}\label{eq:instql}
		L (x, y, h)= (y -  h(x))^2,
	\end{equation}
since  $r_\mu$ is   a  regular  version  of the    conditional expectation $\E_\mu (\Id_{\R}|\Pi_\Xx)$, see  e.g.  \cite[Proposition 1]{CS2001} for a  slight  generalization  of this  fact.
	Hence,  $r_\mu$ is  a minimizer      of   $R^L_\mu: \Meas   (\Xx, [0,1]) \to [0,1]$. Moreover, any minimizer     $R^L_\mu$ to $\Meas   (\Xx, [0,1])$  coincides  with  $r_\mu$  $\mu_\Xx$-a.e.   Thus, the    risk  function  $R^L_\mu$  in     the  generative model $(\Xx, \Yy= \{ 0, 1\}, \Meas(\Xx, \Pp(\Yy)), R^L,  \Pp (\Xx \times \Yy))$  is correct. Hence,  for  any   generative  model $(\Xx, \Yy, \Hh, R^L, \Pp_{\Xx \times \Yy})$, where, abusing notation,  the restriction of $L$ to $\Hh$ is also denoted by $L$, the    risk  function  $R^L_\mu$ is  also correct. In particular the instantaneous 0-1 loss  function  $L^{0,1}: \Xx \times \Yy \times \Meas(\Xx, \Yy),   (x, y, h) \mapsto  d^{0-1} (y, h(x))$, where  $d^{0-1}$ is the 0-1 distance,   generates   a  correct  loss  function.  
\end{example}

\begin{remark}\label{rem:loss} (1) The  class of  generative models of  supervised  learning  of  the type  $(\Xx, \Yy, \Hh, R, \Pp_{\Xx\times\Yy})$ encompasses al  models  for  density  estimation on $\Yy$  by  letting $\Xx$ consist  of a single  point.
	
(2) In  classical statistical learning  theory  one usually  considers    loss  functions $R^L: \Hh\times \Pp_{\Xx \times \Yy} \to \R_{\ge 0} \cup \{ +\infty\}$ which
are  generated  by instantaneous loss  functions  $L: \Xx \times \Yy \times \Hh$, and perturbations  of $R^L$, see, e.g.,  \cite{Vapnik1998}, \cite{Steinwart:SVM2008}.
Note  that any  loss  function which is  generated by an instantaneous   loss function  is empirically  definable.
\end{remark}  
\subsubsection{ RKHSs and correct loss functions}\label{subs:rkhs}
In what  follows  we shall  provide  a natural example of  empirically  definable  correct  loss  functions  which may   not be generated  by   instantaneous loss functions and their perturbations, using  our characterization of  regular  conditional probability measures (Theorem  \ref{thm:marginal})  and   kernel  mean embeddings.

Let  $K: \Yy \times \Yy  \to \R$  be  a  measurable  positive  definite symmetric (PDS)  kernel on  a measurable space $\Yy$.  For $y \in \Yy$ let  $K_y$ be the function on $\Yy$ defined by
$$K_y (y') = K(y, y') \text { for } y' \in \Yy.$$
We denote by $\Hh(K)$ the  associated  RKHS \cite{Aronszajn50}, see also \cite{BT2004} i.e.
$$ \Hh (K)  =\overline{{\rm span}}\{ K_y, y \in \Yy\},$$
where the closure  is taken with respect  to the $\Hh (K)$-norm defined  by
$$\la  K_y, K_{y'} \ra_{\Hh (K)} =  K (y, y').$$
Then for any $f \in \Hh (K)$ we have
\begin{equation}\label{eq:inner1}
	f (y) = \la f, K_y \ra _{\Hh (K)}.
\end{equation}

By Bochner's theorem,  $\int_\Yy \sqrt{K(y,y)}\, d\mu (y) < \infty$  for  $\mu \in \Pp (\Yy)$  if and only  if the   {\it kernel  mean embedding $\mathfrak M_K (\mu) $ of $\mu$}  via the  Bochner integral is well-defined  \cite{BT2004},  where
\begin{equation}\label{eq:kme1}
	\mathfrak M_K (\mu)  =   \int_\Yy  K_y d\mu (y) \in \Hh(K).
\end{equation}
If $\mathfrak M_K(\mu)$ is well-defined  for  all $\mu\in \Pp (\Xx)$,   the  kernel  mean embedding $\mathfrak M_K :\Pp (\Xx) \to \Hh (K) $  extends   to    a linear  map, also denoted  by 	$\mathfrak M_K$ from $\Ss (\Yy)$ to $\Hh (K)$.

\begin{example}\label{ex:losscatkernel} (cf. \cite[Example 3.4 (2)]{Le2023})  Assume $K$ be a  measurable  positive  definite symmetric (PDS)  kernel on  a measurable space $\Xx \times \Yy$  such that $\mathfrak M_K: \Pp (\Xx\times \Yy) \to \Hh (K)$ is an embedding. Then  the  loss  function
\begin{equation}\label{eq:losscatkernel}
R^K: \Probm (\Xx, \Yy) \times \Pp (\Xx, \Yy) \to \R_{\ge 0},  (h, \mu)\mapsto  \| \mathfrak M_K (\Gamma_h)_*\mu_\Xx - \mathfrak M_k (\mu)\|_{\Hh (K)}
\end{equation}
is a correct  loss  function by Theorem \ref{thm:marginal}.

It is  known that if $\Xx \times \Yy$  is a Polish subspace in $\R^n$ then  there  are many  PDS kernels   on  $\Xx \times \Yy$ such  that  $\mathfrak M_K: \Pp (\Xx\times \Yy) \to \Hh (K)$ is an embedding \cite[Theorem 3.2]{Sriperumbudur16}, e.g.  $K$  is the  restriction  of the  Gaussian  kernels     $K_\sigma: \R^{n+m} \times \R^{n+m} \to \R_{\ge 0}  $ for $\sigma >0:  \, K_\sigma (z, z') = \exp (-\sigma \| z-z'\|  ^2)$.  Here  $\|\cdot \|$  denotes  the  norm  in $\R^{n +m}$ generated by the Euclidean metric.
\end{example}

\begin{remark}\label{rem:losscat} (1) If $\Xx$ consists  of a single  point, the loss function  $R^K$   was used    by  Lopez-Paz,  Muandet,  Sch\"olkopf, and  Tolstikhin  for   probability measure  estimation  problem. They proved the following beautiful result on estimating probab.
	
\begin{proposition}\label{prop:lmst2015}\cite[Theorem 1]{LMST2015}  Let $K: \Yy \times \Yy \to \R$  be a measurable  kernel such  that  $\mathfrak M_K: \Ss(\Yy) \to \Hh (K)$ is well-defined.  Assume that  $\| f\| _\infty \le 1$  for all $f \in \Hh(K)$ with $\| f\|_{\Hh (K)} \le 1$. Then   for  any $\eps \in (0,1)$ we have
	\begin{align}
		\mu\Big\{ S_n \in \Yy^n:\, \| \mathfrak M_K (\mu_{S_n})-\mathfrak M_K(\mu)\| _{\Hh (K) }\le  2 \sqrt{ \frac{ \int_\Yy K (y, y) d\mu (y) }{n}}\nonumber\\ + \sqrt{ \frac {2 \log \frac{ 1}{\eps}}{n}}\Big\} \ge  1 -\eps.\label{eq:lmst2015}
	\end{align}
\end{proposition}

(2)  For more examples  of correct  loss   functions, see \cite[Example 3.4, Theorem 4.6]{Le2023}.
\end{remark}

\subsection{Leanability  of overparameterized   supervised  learning models}\label{subs:over}
In this  subsection we  shall  specify the  notion  of successful  learning    of a   procedure  of   statistical learning from data  as the concept  of learnability of a  generative model supervised  (Definition \ref{def:genbi}) learning  and present examples of  learning  algorithms (Definition \ref{def:aserm}).
Then we   demonstrate   an  application  of    categorical  and geometrical methods  in proving  the   learnability of  overparameterized   supervised  learning models.

\subsubsection{Learnability  of a  supervised  learning model  and regularized  ERM algorithms}


Given a   supervised learning model $(\Xx\times \Yy, \Hh, R, \Pp_{\Xx\times \Yy})$   and $\mu \in \Pp_{\Xx\times \Yy}$, we set
\begin{equation}\label{eq:apperror}
	R_{\mu, \Hh} : = \inf_{h \in \Hh}R_\mu  (h).
\end{equation}
For $h \in \Hh$, recall  that  its  {\it estimation  error} is defined as  follows (cf. \cite[Chapter I, \S 3]{CS2001}):
\begin{equation}\label{eq:aprerr}
	\Ee_{\Hh, R, \mu} (h): = R_\mu  (h) - R_{\mu, \Hh}.
\end{equation}

If $R = R^L$ we shall  write  $\Ee_{\Hh, L, \mu}$ instead  of 
$\Ee_{\Hh, R^L, \mu}$. Let $\mu_*$  denote  the  inner  measure defined  by $\mu \in  \Mm (\Xx)$.

\begin{definition}\label{def:genbi} cf. \cite[Definition 3.8]{Le2023}    A  supervised   learning  model
	$(\Xx \times \Yy, \Hh, R, \Pp_{\Xx \times \Yy})$  will be  said  to have  a {\it generalization  ability} or  will be   called  {\it learnable}, if  there   exists  a  uniformly  consistent learning algorithm  
	$$ A: \bigcup_{n=1}^\infty (\Xx \times \Yy)^n \to  \Hh, $$
	i.e.,  for any  $(\eps, \delta)\in (0,1)^2$  there  exists   a number $m_A  (\eps, \delta)$ such  that   for any $m \ge m_A (\eps, \delta)$ and any $\mu \in \Pp_\Zz$ we have
	\begin{equation}\label{eq:pac}
		(\mu^n)_*  \{ S \in (\Xx \times \Yy)^n, \Ee_{\Hh, R, \mu} (A(S)) \le \eps\} \ge 1- \delta.
	\end{equation}
	In this  case  $A$  will be called {\it uniformly consistent}.
\end{definition}

\begin{remark}\label{rem:gen}  (1) We may consider   a general  statistical learning  model $(\Zz, \Hh, R, \Pp_\Zz)$ where  $\Zz$ is a   measurable  (sample) space, $\Hh$ is a  class  of   mappings,  $\Pp_\Zz\subset \Pp (\Zz)$  is  the  set of  all possible  probability measures  governing  distribution  of  observable data  $z_i \in \Zz$  and $R: \Hh \times \Pp_\Zz \to \R\cup {+\infty}$ is a loss  function.
	
	 (2)The   current  definition of      uniform consistency, also called   generalizability, of  a learning algorithm $A$ in the standard literature  is   almost  identical  to our definition.   As in the theory of empirical processes \cite{VW2023},  we relax the  convergence in probability in the classical  requirement   to the convergence   in outer probability  for  the sequence  of functions $\Ee_{\Hh, R, \mu} \circ A: (\Xx \times \Yy)^n \to \R, \, n \in \N$, whose  $\mu^n$-measurability needed  in  the definition of  the  convergence in probability  represents a  quite strong  assumption.
\end{remark}

From now on we  assume that  the loss function $R: \Hh \times \Pp_{\Xx \times \Yy} \to \R \cup \{ +\infty\}$ is empirically   definable.   For  a      data $\{S_n\in (\Xx \times \Yy)^n\}$ we define  the  empirical risk
$$\widehat  R_{S_n}: \Hh \to  \R, \, h \mapsto   R_{\mu_{S_n}} (h). $$

It is well-known  in analysis  that  to solve     an equation  $A(f) =  F$ it is often   useful to look  at its perturbed  equation   $A_\eps (f) = F_\eps$.  A successful perturbation method  has  been  used  in   approximation theory  under the  name    of  Tikhonov's regularization method, which  has been  developed  further  by   Vapnik-Stephaniuk  as  a method  of solving  stochastic  ill-posed  problem \cite{VS1978}, see also \cite[Chapter 7]{Vapnik1998}, \cite[Chapter 7]{Vapnik2000}.  The Vapnik-Stephaniuk  method of  solving stochastic  ill-posed problem fits  particularly  well to   our equation  for    regular  conditional     probability measures (Theorem \ref{thm:marginal})  which leads  to  the  correct  loss  function in  Example \ref{ex:losscatkernel}.

\begin{definition}\label{def:aserm}Assume that   $(\Xx \times \Yy, \Hh,R, \Pp_{\Xx\times \Yy} )$ is a   generative model ofs supervised  learning.  Let   $W: \Hh \to \R_{\ge 0}$  be a function.  Given  a sequence  of numbers $\Gamma : = (\gamma_1 > \gamma_2 > \ldots >  0: \, \lim _{n  \to \infty}\gamma_n = 0)$, we shall call  a  sequence  of   the following  regularized loss functions $R_{\gamma_n}$ for $R$
\begin{equation}\label{eq:regularized}
	R_{\gamma_n}: \Hh \times \Pp _{\Xx \times \Yy} \to \R \cup \{ +\infty\}, (h, \mu) \mapsto   R (h, \mu) + \gamma _n  W (h)
\end{equation}
 a  sequence of {\it $(W,\Gamma)$-regularized   loss  functions  for $R$}. 
If  $\gamma_i = 0$ for   all $i$  we   write $\Gamma= 0$.
Let  $C = (c_1\ge  \ldots\ge  c_n\ge  \ldots: \, c_i \ge 0)$.    A learning  algorithm
	$$ A: \bigcup _{n \in \N} (\Xx \times \Yy)^n \to  \Hh $$
	will be called  a {\it $(C, \Gamma)$- regularized empirical  risk minimizing} algorithm, abbreviated  as  $(C,\Gamma)$-regularized ERM algorithm,  if  there  exists a function $W: \Hh \to \R_{\ge 0}$   such that	 for  any $n \in \N$  and any $S \in (\Xx\times \Yy)^n$, we have
	$$R_{\gamma_n}(A (S), \mu_S)  - \inf _{ h\in  \Hh}  R_{\gamma_n} (h, \mu_S) \le c_n. $$
	If  $c_i = 0$ for   all $i$  we   write $C= 0$,
\end{definition}

\begin{remark}\label{rem:erm} (1) $(C,\Gamma)$-regularized ERM algorithms are  the most  frequently  used  algorithms  in statistical learning  theory \cite{Vapnik1998}.  If $\Gamma = 0$, we  shall write  $C$-ERM algorithm instead  of $(C,0)$-regularized ERM algorithm. In this  case      a specification of $C$  is  often  suppressed   \cite[p. 80]{Vapnik1998}.  In general cases, we need  to specify $C$  to ensure   that a $(C,\Gamma)$-regularized ERM algorithm is uniformly consistent.

(2) The  uniform  consistency of   $C$-ERM  algorithms,  which  is also called ERM-algorithms by suppressing $C$, in the  case  that   a loss  function $R =  R^L$ is generated  by  an instantaneous  loss  function $L : \Xx \times \Yy \times \Pp (\Xx \times \Yy) \to \R_{\ge 0}$ and  $\Pp_{\Xx \times \Yy} = \Pp (\Xx \times \Yy)$ has been discussed  in depth  by Vapnik \cite{Vapnik1998}, \cite{Vapnik2000}.  Vapnik  emphasized the relation  between  the consistency  of ERM  algorithms  and the  convergence   in probability  of  associated  empirical  processes \cite[\S 3.3, p. 86]{Vapnik1998}, which could  be regarded  as  a uniform  law of large numbers in a functional space \cite[\S 3.3.1, p. 87]{Vapnik1998}.  

Berner-Groh-Kutyniok-Peterson  note  that  Vapnik's theory  for ERM algorithms   with  respect to  distribution free  models, i.e.,  where $\Pp_{\Xx \times \Yy} = \Pp(\Xx \times \Yy)$, is not sufficient  to explain  the learnability  of   neural networks,  and they asked  for a new  statistical  learning theory, which  must  take  into  account  also  geometry of  
underlying statistical model $\Pp_{\Xx \times \Yy}$  \cite[\S 1.1.3]{BGKP2021}.

Our examples  of overparametrized   supervised  learning  models  (Proposition \ref{prop:learnable}) in  the next subsection   satisfy  these    requirements, since  our theorem takes into account the interplay between the geometry of a statistical model $\Pp_{\Xx \times \Yy}$ and the geometry of  a hypothesis  space  $\Hh \subset  C_{Lip}(\Xx, \Pp(\Yy)_{\tilde  K_2})$.
\end{remark}

\subsubsection{Examples  of  overparameterized   supervised  learning  models}

\

$\bullet$ Let  $\Xx $ be a compact subset in $\R^n \times \{0\}\subset \R^{n +m}$   and $\Yy$ a     compact subset  in $\{0\}\times \R^m$.\\ 
$\bullet$   Let $K_2: \Yy \times \Yy \to \R $ be  the restriction  of a   Gaussian  kernel $K = K_\sigma:\R^{m +n} \times \R^{m+n} \to \R_{\ge 0}$  to $\Yy \times 
\Yy$.  Let  $K_1: \Xx \times \Yy  \to \R$  be  the restriction  of the kernel $K$ to  $(\Xx \times \Yy)$.\\
$\bullet$ It is known   that  the kernel mean  embeddings $\mathfrak M_{K_1}: \Pp (\Xx \times \Yy) \to  \Hh (K_1)$   and
$\mathfrak M_{K_2}: \Pp (\Yy) \to \Hh (K_2)$  extend linearly to  embeddings,  denoted by  $\mathfrak M_{K_1}: \Ss(\Xx \times \Yy)\to \Hh(K_1) $ and , $\mathfrak M_{K_2}: \Ss(\Yy)\to \Hh (K_2)$  respectively. We denote by $\| \cdot \|_{\tilde K_1}$  and   $\| \cdot \|_{\tilde K_2}$ the pull-back   of the  norm  $\|\cdot \| _{\Hh (K_1)}$  on $\Ss(\Xx \times \Yy)$ and  on $\Ss (\Yy)$ respectively. It  is known  that  the restriction  of  $\| \cdot \|_{\tilde K_1}$   to $\Pp (\Xx \times \Yy)$ and  the restriction  $\| \cdot \|_{\tilde K_2}$ to  $\Pp (\Yy)$   generate metrics  that  induce  the  weak*-topology on   the   spaces, \cite[Example 1, Theorem 3.2]{Sriperumbudur16}.
We  denote  by $\Pp (\Xx \times \Yy)_{\tilde {K_1}}$  and $\Pp (\Yy)_{\tilde {K_2}}$  the  corresponding metric spaces.\\
 $\bullet$ For metric  spaces $(F_1, d_1)$, $(F_2, d_2)$, we  denote by $C_{Lip}(F_1, F_2)$  the space of all  Lipschitz  continuous  mappings  from $F_1$ to $F_2$.  For  $h \in  C_{Lip}(F_1, F_2)$  we denote by $L(h)$ the  {\it Lipschitz   constant}  of $h$  namely the  nonnegative number
 $$L(h):= \sup _{x\not =y} \frac{d_2 (h(x), h(y))}{d_1(x,y)}.$$\\
$\bullet$ Denote by $\Pp_{Lip}(\Xx, \Pp(\Yy)_{\tilde K_2}, vol_\Xx)$  the set  of all   probability measures $\mu \in \Pp (\Xx \times \Yy)$  such that:\\
(i) ${\rm  sppt}\,\mu_\Xx  = \Xx$,  where $\mu_\Xx = (\Pi_\Xx)_*\mu$; \\
(ii)  there  exists a   regular   conditional   measure $\mu_{\Yy|\Xx} \in C_{Lip} (\Xx, \Pp (\Yy)_{\tilde K_2})$ for $\mu$  with respect to  the projection $\Pi_\Xx: \Xx \times \Yy \to \Xx$. 

 We define   a  loss function (cf. Example \ref{ex:losscatkernel})
\begin{align}\label{eq:lostvk}
R^{K_1}:  C_{Lip}(\Xx, \Pp(\Yy)_{\tilde K_2})\times \Pp_{Lip} (\Xx, \Pp(\Yy)_{\tilde K_2}, vol_\Xx)\to \R_{\ge 0},\nonumber\\
  (h, \mu) \mapsto \| (\Gamma_h)_*\mu_\Xx -\mu\|_{\tilde K_1}.
\end{align}

\begin{proposition}\label{prop:learnable} (\cite[Corollary 6.3]{Le2023})   Assume that  $\Pp_{\Xx \times \Yy} \subset\Pp_{Lip} (\Xx, \Pp(\Yy)_{\tilde K_2}, vol_\Xx)$ satisfies   the following  condition (L).
	
	(L) $\Pp_{\Xx\times \Yy}$ is compact in the weak*-topology and  the     function  $(\Pp_{\Xx \times \Yy})_{\tilde K_1} \to \R, \mu \mapsto    L (\mu_{\Yy|\Xx})$, where  $\mu_{\Yy|\Xx} \in C_{Lip} (\Xx, \Pp (\Yy)_{\tilde K_2})$,    takes  values  in a  finite interval $[a, b] \subset \R$.
	
	(1) Then    there  exists  a $(C, \Gamma)$-regularized  ERM algorithm  $A$ for  the loss  function  $R^{K_1}$ which   is uniformly consistent  for the  supervised  learning  model $(\Xx, \Yy, C_{Lip} (\Xx, \Pp (\Yy)_{\tilde K_2}), \Pp_{\Xx \times \Yy})$.
	
	(2)  Assume further that  $\Hh$ is a    subspace   of  the space $C_{Lip}(\Xx, \Yy)$. Then  the   supervised  learning  model $(\Xx, \Yy, \Hh, R^{K_1}, \Pp_{\Xx \times \Yy})$  has a generalization ability. 
\end{proposition}

\begin{remark}\label{rem:learnable}  Here  is a concrete  example  of $\Xx, \Yy$ and $\Pp_{\Xx\times \Yy}\subset\Pp_{Lip} (\Xx, \Pp(\Yy)_{\tilde K_2}, vol_\Xx)$ satisfying  the  condition (L)  in Proposition  \ref{prop:learnable}.   Let $\Xx = [0, 1]^n \subset \R^n \times \{0\} \subset \R^{n +m}$ and $\Yy = [0, 1]^m \subset \{ 0\}\times \R^m \subset \R^{n +m}$. Denote by $dx$ the  restriction of the  Lebesgue  measure on $\R^n$  to $\Xx$ and  by $dy$  the restriction of the Lebesgue measure  on $\R^m$  to $\Yy$. Let   $\Pp_{\Xx \times \Yy}$  consist  of  all  probability measures   $\mu_f: =f dxdy$  such that there  exists $c_1, c_0 >0$ with the following  property:  $f\in C^1 (\Xx\times \Yy)$, $L(f) \le  c_1$, and $c_1 \ge f (x, y) \ge c_0 $  for all $x, y \in \Xx \times \Yy$.  
\end{remark}

\emph{Outline of  the  proof  of Proposition \ref{prop:learnable}}

The proof of  Proposition \ref{prop:learnable} uses   Vapnik-Stefanyuk's  method  of solving  stochastic ill-posed  problem \cite{VS1978}, \cite{Stefanyuk1986}, \cite[Theorem 7.3, p. 299]{Vapnik1998} which has been slightly improved  by L\^e in  \cite[Theorem 6.1]{Le2023}. 
The  main idea  of  Proposition \ref{prop:learnable} uses our   characterization  of   regular  conditional  probability measure in Theorem \ref{thm:marginal} in a  variational form, namely,  the loss  function  $R^{K_1}_\mu: \Hh \to \R$, where  $\mu\in \Pp (\Xx \times \Yy)$ is   the (unknown)   joint distribution of $x, y$.  Since  $\mu$ is unknown,   we   wish to  use   the  empirical  measures $\mu_{S_n}$ where $S_n \in (\Xx \times \Yy)^n$  are observed data to approximate $\mu$ in the definition of   the loss  function $R_\mu: \Hh \to \R_{\ge 0}$. But such a direct  approximation  method  may not work  and we have  to perturb the loss  function  $R^{K_1}_\mu: \Hh \to \R$  to  a  sequence  of $(W, \Gamma)$-regularized  loss functions.  Vapnik-Stefanyuk's  method and its  variant \cite[Theorem 6.1]{Le2023}  provide  conditions  when a  $(C, \Gamma)$- ERM algorithm is  consistent.

So let us first give  a  general setting   of Vapnik-Stefanyuk's  methods  of  solving  stochastic  ill-posed  problems.

We consider the following  operator  equation
\begin{equation}\label{eq:vapnik71}
	Af = F
\end{equation}
defined  by a  continuous  operator $A$ which  maps  in a one-to-one manner  the elements  $f$  of a metric  space $(E_1, \rho_{E_1})$ into  the  elements  of a metric
space $(E_2, \rho_{E_2})$,  assuming that 
a solution  $f \in  E_1$  of \eqref{eq:vapnik71}  exists  and is unique. 

Assume  that    $A$ belongs  to a   space  $\Aa$ and  instead  of  Equation \eqref{eq:vapnik71}  we are given a   sequence $\{ F_{S_l} \in E_2, \, l \in \N^+\}$,  a sequence  $\{ A_{S_l} \in \Aa, l \in \N^+ \}$, where $S_l$ belongs  to a  probability space $(\Xx_l,\mu_l)$ and  $A_{S_l}, F_{S_l}$   are   defined  by a family of   maps $\Xx_l \to E_2,\,  S_l \mapsto F_{S_l},$  and  $\Xx_l \to \Aa,\, S_l \mapsto A_{S_l}$.

Let $W:  E_1 \to  \R_{\ge 0}$  be a lower semi-continuous     function   that satisfies  the   following  property (W).

(W)  The sets  $\Mm_c =  W^{-1} ([0,c])$ for  $c \ge  0$ are all  compact.

Given  $A_{S_l}, F_{S_l}$, and $\gamma_l >0$, let us      define  a  regularized risk function $R^*_{\gamma_l} (\cdot, F_{S_l}, A_{S_l}): E_1 \to \R$  by
\begin{equation}\label{eq:regl}
	R^*_{\gamma_l} (\hat f, F_{S_l}, A_{S_l}) = \rho_{E_2}^2  (A_{S_l}\hat f, F_{S_l}) + \gamma_l W(\hat f).
\end{equation}
We shall say   that  $f_{S_l}\in E_1$  is  an {\it $\eps_l$-minimizer}    of  $R^*_{\gamma_l}$ if
\begin{equation}\label{eq:regel}
	R^*_{\gamma_l} (f_{S_l}, F_{S_l}, A_{S_l}) \le  R_{\gamma_l} (\hat f, F_{S_l}, A_{S_l}) + \eps_l \text{ for   all } \hat f \in  E_1.
\end{equation}
We  shall also use the shorthand notation $A_l$ for  $A_{S_l}$, $F_l$ for  $F_{S_l}$, $f_l$ for $f_{S_l}$,
$\rho_2$ for $\rho_{E_2}$, $\rho_1$ for  $\rho_{E_1}$.
For any $\eps_l >0$,  an $\eps_l$-minimizer  of  $R^*_{\gamma_l}$  exists.
We will measure  the closeness  of an operator  $A$ and an operator $A_l$ by the distance
\begin{equation}\label{eq:v712}
	\|A_l - A\| := \sup_{\hat f \in E_1}\frac{\|A_l \hat f -A\hat f\|_{E_2}}{W^{1/2}(\hat f)}.
\end{equation}

The following theorem    proved by L\^e  in \cite[Theorem 6.1]{Le2023} is a slight  improvement  of  Stefanyuk's theorem \cite{Stefanyuk1986}, \cite[Theorem 7.3, p. 299]{Vapnik1998}.

\begin{proposition}\label{thm:61Le2023} cf. \cite[Theorem 7.3, p. 299]{Vapnik1998} Let $f_{S_l}$ be  a $\gamma_l^2$-minimizer of $R^*_{\gamma_l}$ in \eqref{eq:regl}  and  $f$  the  solution of \eqref{eq:vapnik71}.    For any  $\eps >0$ and any constant $C_1, C_2  >0$ there
	exists a value $\gamma(\eps, C_1, C_2) >0$ such that    for any $\gamma_l \le \gamma(\eps, C_1, C_2)$
	\begin{eqnarray}\label{eq:vapnik713}
		(\mu_l)^* \{S_l\in \Xx_l:\, \rho_1 (f_{S_l}, f) >\eps \}\le (\mu_l)^* \{S_l \in \Xx_l:\, \rho_2  (F_{S_l}, F) > C_1 \sqrt{\gamma_l}\}\nonumber \\ + (\mu_l)^* \{ S_l \in \Xx_l:\, \|A_{S_l} - A\| > C_2\sqrt{\gamma_l} \}
	\end{eqnarray}
	holds true.
\end{proposition}

In the first  step of  proof  of Proposition \ref{prop:learnable}, we shall find a  condition for a  generative model  of supervised  learning  together with the perturbation term  $W$  such that the  conditions of Proposition \ref{thm:61Le2023}  are  met.   This   first  step of defining $W$ has been    done  using a new  distance   $d_M$  on $C(\Xx, \Mm (\Yy)_{\tilde K_2})  $   as follows:

\begin{equation}\label{eq:dM}  
d_M (f, f') : = \sup_{x\in \Xx} ( \|(f - f')(x)\|_{\tilde K_2} + \|\Gamma _f (x)  -\Gamma_{f'} (x)\|_{\tilde K_2}
\end{equation}
In other worlds, the metric $d_M$ is induced by the norm $\| \cdot \| _{M}$ on the space $C (\Xx,  S(\Yy)_{\tilde K_2})$  given  as follows
$$\| f\|_{ M}  =  \sup  _{x \in \Xx} (\| f(x)\|_{\tilde  K_2} +\|\Gamma_f (x)\|_{\tilde K_1}).$$

\begin{lemma}\label{lem:W}  Let  $W: C_{Lip}(\Xx, \Pp (\Yy)_{\tilde K_2})_M \to \R \ge _{0}$   be defined as follows
	\begin{equation}\label{eq:lsc}
	W (f): =  \| f\|_M  + L(f)  +\| \Gamma_{\underline f}\|_{\tilde K_3, \tilde K_1}
	\end{equation}
  Then $W: C_{Lip}(\Xx, \Pp (\Yy)_{\tilde K_2})_M  \to \R_{\ge 0}$ is a lower  semi-continuous  function. Furthermore, for  any $c\ge 0$ the set
	$W^{-1}[0,c]$ is a compact set in $C_{Lip}(\Xx, \Pp (\Yy)_{\tilde K_2})_M$.
\end{lemma}

For the proof of Lemma \ref {lem:W}  we refer to  \cite[Proposition 6.1]{Le2023}.

Using  Lemma  \ref{lem:W}  and  Proposition \ref{thm:61Le2023}, we  obtain  the following

\begin{proposition}\label{thm:mainLe2023}
	Let  $\Xx $ be a compact subset in $\R^n \times \{0\}\subset \R^{n +m}$   and $\Yy$ a     compact subset  in $\{0\}\times \R^m$.   Let $K_2: \Yy \times \Yy \to \R $ be  the restriction  of  the Gaussian kernel  $K: \R^{m+n}\times \R^{m+n} \to \R$.
	Denote by $\Pp_{Lip}(\Xx, \Pp(\Yy)_{\tilde K_2}, vol_\Xx)$  the set  of all   probability measures $\mu \in \Pp (\Xx \times \Yy)$  such that:\\
	(i) ${\rm  sppt}\,\mu_\Xx  = \Xx$,  where $\mu_\Xx = (\Pi_\Xx)_*\mu$; \\
	(ii)  there  exists a   regular   conditional   measure $\mu_{\Yy|\Xx} \in C_{Lip} (\Xx, \Pp (\Yy)_{\tilde K_2})$ for $\mu$  with respect to  the projection $\Pi_\Xx: \Xx \times \Yy \to \Xx$. 
	
	Let  $K_1:\Xx \times \Yy)\times (\Xx\times \Yy)  \to \R$  be  the restriction  of the kernel $K$ to  $(\Xx \times \Yy)$. We define   a  loss function
	$$	R^{K_1}:  C_{Lip}(\Xx, \Pp(\Yy)_{\tilde K_2})\times \Pp_{Lip} (\Xx, \Pp(\Yy)_{\tilde K_2}, vol_\Xx)\to \R_{\ge 0},  (h, \mu) \mapsto \| (\Gamma_h)_*\mu_\Xx -\mu\|_{\tilde K_1}$$
	as in \eqref{eq:lostvk}.
	Then   for any  $\mu \in \Pp_{Lip}  (\Xx, \Pp(\Yy)_{\tilde K_2}, vol_\Xx)$, there  exists a  consistent    $(C,\Gamma)$-regularized ERM algorithm  
	$A$  for  the supervised  learning model\\ $(\Xx, \Yy, C_{Lip} (\Xx, \Pp (\Yy)_{K_2}),R^{K_1},  \Pp_{Lip} (\Xx, \Pp(\Yy)_{\tilde K_2}, vol_\Xx))$. Moreover,  for  any $\eps, \delta > 0$  there exists  $N(\eps, \delta)$ such that  for any $n \ge N(\eps, \delta)$ we have
	\begin{equation}\label{eq:van1}
		(\mu ^n)^*\{ S_n \in  (\Xx\times \Yy)^n :  d_M (A(S_n),\mu_{\Yy|\Xx}) > \eps \} \le \delta,
	\end{equation} 
	where $\mu_{\Yy|\Xx}\in C_{Lip}(\Xx, \Pp(\Yy)_{\tilde K_2}$ is  the unique  regular conditional  probability measure  for $\mu$ with respect to the  projection $\Pp_\Xx: \Xx \times \Yy \to \Xx$. 
\end{proposition}

For a detailed  proof  of Proposition \ref{thm:mainLe2023}, we refer to  \cite[Theorem 6.5]{Le2023}.  

Proposition \ref{prop:learnable}    is derived  from Proposition \ref{thm:mainLe2023}.  Under the condition (L) of Proposition \ref{prop:learnable}, $W(f)$ is bounded. Hence we  can bound  the  RHS of \eqref{eq:vapnik713} in the  condition of  Proposition \ref{thm:mainLe2023} by using   the estimation  in \eqref{eq:lmst2015}. It is not hard to see  that     a $(C, \Gamma)$ regularized algorithm using the regularized risk function $R^*_{\gamma_n}$ in \eqref{eq:regel} in the  condition of  Proposition \ref{thm:mainLe2023} is uniformly consistent, if $\lim_{n \to \infty} \gamma_n = 0$ and  $\lim_{n \to \infty} n \gamma_n = \infty$. 
 Note  that  the difference between Proposition \ref{prop:learnable}  and Proposition \ref{thm:mainLe2023}  is that  in Proposition \ref{thm:mainLe2023}   we   choose a  subset  $\Hh \subset C_{Lip}(\Xx, \Yy) \subset  C_{Lip} (\Xx, \Pp (\Yy)_{\tilde K_2})$; moreover,  we require  the condition (L).  
 Since  $\Hh \subset C_{Lip}(\Xx, \Yy) \subset  C_{Lip} (\Xx, \Pp (\Yy)_{\tilde K_2})$, using the explicit form  of the loss  function $R^{K_1}$   we can find    a  learning  algorithm  for  the supervised  learning  model    $(\Xx, \Yy,\Hh, R^{K_1}, \Pp_{\Xx \times \Yy})$ from the one for the larger  model  $(\Xx, \Yy,  C_{Lip} (\Xx, \Pp (\Yy)_{\tilde K_2}), R^{K_1}, \Pp_{\Xx \times \Yy})$, which is ensured by Proposition \ref{thm:mainLe2023} and  the condition (L). 
 
   Our examples  in Proposition \ref{prop:learnable}  are of a different  nature  than  the consistent  non-parametric  regression  estimators     described in \cite{Tsybakov2009}.  In \cite[\S 7.9.1, Theorem 5.7, p. 2.5.2 ]{Vapnik2000} Vapnik  also  suggested  a sufficient condition on the function $W$  for the uniform consistency  of   the $(C,\Gamma)$-algorithm in   Stefanyuk's theorem  \cite{Stefanyuk1986}, \cite[Theorem 7.3, p. 299]{Vapnik1998}. His condition resembles  the sufficient condition  for  the learnability in Proposition \ref{prop:learnable}.

\section{Geometric kernels  and applications}\label{sec:geomkernel}

Kernel methods form a very important paradigm in modern machine learning.
In the literature, most positive definite kernels are defined over inner product spaces.
However, in many applications, including brain computer interfaces, radar signal processing, and computer vision, the actual data can have much richer geometrical structures beyond the Euclidean space setting, including in particular those of a smooth manifold. 
There are thus both mathematical and practical interests to investigate positive definite kernels over non-Euclidean structures. 
In this section, we present several results on positive definite kernels defined over the set of SPD (symmetric positive definite) matrices
along with corresponding generalizations to the set of positive Hilbert-Schmidt operators on a Hilbert space. 

Consider first the general setting of a metric space $(M,d)$. It is natural to investigate whether the generalization of the Gaussian kernel
$K_\gamma: \R^n \times \R^n \mapto \R$, $K(x,y) = \exp(-\gamma ||x-y||^2)$, $\gamma > 0$, to
this setting, namely $K_\gamma:M \times M \mapto \R$, $K(x,y) = \exp(-\gamma d^2(x,y))$, is positive definite.

We first briefly review the concepts of positive definite and negative definite kernels, see e.g. \cite{Berg1984}, Chapter 3.
Let $X$ be a nonempty set. A function $\varphi: X \times X \mapto \R$ is said to be a {\it positive definite kernel} if and only if
it is symmetric and
\begin{align}
	\sum_{j,k=1}^Nc_jc_k\varphi(x_j,x_k) \geq 0
\end{align}
where $N \in \Nbb$, $\forall \{x_j\}_{j=1}^N \subset X$, $\{c_j\}_{j=1}^N \subset \R$. 
A function $\varphi: X \times X \mapto \R$ is said to be a {\it negative definite kernel} if and only if it is symmetric and
\begin{align}
	\sum_{j,k=1}^Nc_jc_k\varphi(x_j,x_k) \leq 0
\end{align}
where $N \geq 2$, $\forall \{x_j\}_{j=1}^N \subset X$, $\{c_j\}_{j=1}^N \subset \R$, with $\sum_{j=1}^Nc_j = 0$.

Positive definite and negative definite kernels are closely related. The following is Theorem 3.2.2 in \cite{Berg1984}, a generalization of
Theorem 1 in \cite{Schoenberg1938:Positive}, which is stated for $\varphi(x,y) = d^2(x,y)$ on a separable semi-metric space $(M,d)$ (see Theorem \ref{theorem:Gaussian-kernel-metric-space} below):
\begin{theorem}
	\label{theorem:positive-negative-definite}
	Let $X$ be a nonempty set. Then $\varphi: X \times  X \mapto \R$ is negative definite if and only if $\exp(-\gamma \varphi)$ is positive definite $\forall \gamma > 0$.
\end{theorem}

In particular, on a Hilbert space $(\Hcal, \la,  \ra, ||\;||)$, it can be readily verified that the kernel 
$\varphi(x,y) = ||x-y||^2$ is negative definite, since the condition $\sum_{j=1}^N c_j = 0$ implies that
$\sum_{j,k=1}^Nc_jc_k||x_j-x_k||^2 = -||\sum_{j=1}^Nc_jx_j||^2 \leq 0$ always. Thus the Gaussian kernel
$K_\gamma(x,y) = \exp(-\gamma ||x-y||^2)$ is positive definite $\forall \gamma > 0$.
Furthermore, it can be shown that $K(x,y) = \exp(-\gamma ||x-y||^p)$ is positive definite $\forall \gamma > 0$ if and only if $0 < p \leq 2$
(\cite{Schoenberg1938:Positive}, Corollary 3).

More generally, let $X$ be a nonempty set and suppose that there exists a map $\psi: X \mapto \Hcal$.
Then the kernel $\varphi: X \times X \mapto \R$ defined by $\varphi(x,y) = ||\psi(x) - \psi(y)||^2$ is negative definite, hence $K: X \times X \mapto \R$, $K(x,y) = \exp(-\gamma ||\psi(x)-\psi(y)||^2)$ is positive definite $\forall \gamma > 0$.

In general, on a metric space $(M,d)$, by Theorem \ref{theorem:positive-negative-definite}, the kernel $K_\gamma(x,y) = \exp(-\gamma d^2(x,y))$ is positive definite $\forall \gamma > 0$ if and only if $\varphi(x,y) = d^2(x,y)$ is negative definite. Schoenberg \cite{Schoenberg1938:Positive} proved that this happens if and only if $(M,d)$ is isometrically embeddable into a Hilbert space $\Hcal$. The following is essentially Theorem 1 in \cite{Schoenberg1938:Positive}, see also Proposition 3.3.2 in \cite{Berg1984} for a more general version.

\begin{theorem}
	\label{theorem:Gaussian-kernel-metric-space}
	Let $(M,d)$ be a metric space. The kernel $K_\gamma:M \times M \mapto \R$,
	$K_\gamma(x,y) = \exp(-\gamma d^2(x,y))$, is positive definite $\forall \gamma > 0$, or equivalently, the kernel $\varphi(x,y) = d^2(x,y)$ is negative definite, if and only if there exists a Hilbert space $\Hcal$ and a map $\psi: M \mapto \Hcal$ such that $d(x,y) = ||\psi(x) - \psi(y)||$.
\end{theorem}

Consider now the setting where $(M,g)$ is a geodesically complete finite-dimensional Riemannian manifold. Let $d_g$ denote the corresponding 
metric on $M$ induced by $g$.
Then $(M,d_g)$ is a complete metric space. Building upon Theorem \ref{theorem:Gaussian-kernel-metric-space}, we have the following result (see \cite{Jayasumana2015:PAMI}, Theorem 6.2, and \cite{Feragen:2015Geodesic}, Theorem 2)

\begin{theorem}
	\label{theorem:Gaussian-kernel-geodesic-Riemannian-manifold}
	Let $(M,g)$ be a geodesically complete finite-dimensional Riemannian manifold. The kernel 
	$K_\gamma:M \times M \mapto \R$,
	$K_\gamma(x,y) = \exp(-\gamma d^2(x,y))$, is positive definite $\forall \gamma > 0$ if and only if $M$ is isometric, in the Riemannian sense, to  Euclidean space $\R^n$, for some $n \in \Nbb$.
\end{theorem}

Thus on a complete Riemannian manifold $(M,g)$, the Gaussian kernel $K_\gamma(x,y) = \exp(-\gamma d^2(x,y))$ is positive definite $\forall \gamma > 0$ if and only if $M$ is isometric to Euclidean space. This means that if $M$ has nonzero curvature, then the Gaussian kernel on $M$ {\it cannot} be positive definite 
$\forall \gamma > 0$. Subsequently,
we apply the above general results to construct positive definite kernels on the sets of SPD matrices
and positive Hilbert-Schmidt operators on a Hilbert space.

\subsection{Kernels defined via Riemannian metrics}
\label{section:kernel-logE}

Throughout the following, let $\Mrm(n)$ denote the set of real
$n \times n$  matrices.  Then $(\Mrm(n),+, \cdot, \la,\ra_F)$ is an inner product space, where $+$ and $\cdot$
denote standard matrix addition and scalar multiplication, respectively, and $\la, \ra_F$ denotes the Frobenius 
inner product, with $\la A,B\ra_F = \trace(A^TB)$.  Subsequently, we refer to $\Mrm(n)$ with this  inner product space structure, which can be identified with the Euclidean space
$(\R^{n^2},\la, \ra)$ under the canonical inner product.
Let $\Sym(n)$ denote the set of real, $n \times n$ symmetric matrices.
Then $(\Sym(n), +, \cdot, \la, \ra_F)$ is a subspace of $\Mrm(n)$. 

Let $\Sym^{++}(n)$ denote the set of real $n \times n$ SPD matrices. Then it is an open convex cone in $\Sym(n)$, being closed under
 scalar multiplication by positive numbers.
Thus $\Sym^{++}(n)$ can be viewed as a smooth manifold, with tangent space $T_P(\Sym^{++}(n)) \cong \Sym(n)$ $\forall P \in \Sym^{++}(n)$, and can be equipped with a Riemannian metric. 

In the following, we investigate the Gaussian kernel defined using the induced Riemannian distance corresponding to three commonly used Riemannian metrics on $\Sym^{++}(n)$, namely the affine-invariant, Bures-Wasserstein, and Log-Euclidean metrics.

{\bf Affine-invariant Riemannian metric}. The most well-known Riemannian metric on $\Sym^{++}(n)$ is the affine-invariant Riemannian metric, the study of which goes as far back as \cite{Siegel:1943} and \cite{Mostow:1955}. This Riemannian metric $g_{\ai}$ is defined by,
$\forall P \in \Sym^{++}(n)$ and $U,V \in T_P(\Sym^{++}(n)) \cong \Sym(n)$,
\begin{align}
	\label{equation:inner-tangent-affine}
	g_{\ai}(P)(U,V) =
	\la P^{-1/2}UP^{-1/2}, P^{-1/2}VP^{-1/2}\ra_F = \trace(P^{-1}UP^{-1}V),
\end{align}
where $\la,\ra_F$ denotes the Frobenius inner product.
It corresponds to the Fisher-Rao metric on the set of zero-mean Gaussian densities on $\R^n$.
The  Riemannian  manifold $(\Sym^{++}(n), g_{\ai})$ is a Cartan-Hadamard manifold, that is 
it is geodesically complete, simply connected, and with nonpositive sectional curvature (see e.g. \cite{Lang:1999Differential}, chapter XII). There is a unique geodesic $\gamma_{\ai}^{AB}$ connecting any pair $A,B \in \Sym^{++}(n)$, 
with closed form expression
\begin{align}
	\label{equation:geodesic-affine-AB}
	\gamma_{\ai}^{AB}(t) = A^{1/2}\exp[t\log(A^{-1/2}BA^{-1/2})]A^{1/2}, \;\; t \in [0,1].
\end{align}
The Riemannian distance between $A$ and $B$ is the length of this geodesic,
\begin{align}
	\label{equation:affine-E}
	d_{\ai}(A,B) =
	||\log(A^{-1/2}BA^{-1/2})||_F,
\end{align}
where $||\;||_F$ denotes the Frobenius norm.
Since $(\Sym^{++}(n),g_{\ai})$
has {\it nonpositive sectional curvature}, for $n \geq 2$, it cannot be isometric to Euclidean space, and thus by Theorem \ref{theorem:Gaussian-kernel-geodesic-Riemannian-manifold}, the kernel $K_\gamma:\Sym^{++}(n) \times \Sym^{++}(n) \mapto \R$ defined by
\begin{align}
	K_\gamma(A,B) = \exp(-\gamma d^2_{\ai}(A,B)) = \exp(-\gamma ||\log(A^{-1/2}BA^{-1/2})||^2_F), \; \gamma > 0,
\end{align}
{\it cannot} be positive definite $\forall \gamma > 0$ (but may be positive definite for some $\gamma > 0$).

{\bf Bures-Wasserstein metric}. Another commonly used Riemannian metric on $\Sym^{++}(n)$ is the Bures-Wasserstein metric  $g_{\bw}$, see, 
e.g., \cite{Takatsu:2011,Bhatia:2018,Malago:2018}.
It corresponds to the $2$-Wasserstein distance between zero-mean Gaussian measures on $\R^n$. The Bures-Wasserstein metric is given by,
$\forall P \in \Sym^{++}(n)$,
\begin{align}
	\label{equation:metric-bures-wasserstein}
	g_{\bw}(P)(U,V) =  \trace(L_P(U)PL_P(V)), \;\; U,V \in \Sym(n),
\end{align}
where $L_P(V) \in \Sym(n)$ is the unique solution of the Lyapunov equation $XP + PX = V$.
$(\Sym^{++}(n), g_{\bw})$ is a Riemannian manifold with {\it nonnegative sectional curvature}.
The Riemannian distance between $A,B \in \Sym^{++}(n)$ is the Bures-Wasserstein distance, which admits the following closed form expression:
\begin{align}
	\label{equation:distance-Bures-Wasserstein}
	d_{\bw}(A,B) = \sqrt{\trace(A) + \trace(B) - 2\trace(A^{1/2}BA^{1/2})^{1/2}}.
\end{align}
It is the length of the geodesic curve
\begin{align}
	\label{equation:geodesic-Bures-Wasserstein}
	\gamma_{\bw}^{AB}(t) = (1-t)^2A + t^2B + t(1-t)[(AB)^{1/2} + (BA)^{1/2}].
\end{align}
Similar to the case of the affine-invariant metric, $(\Sym^{++}(n), g_{\bw})$ {\it cannot} be isometric to Euclidean space for $n \geq 2$.
Thus, by Theorem \ref{theorem:Gaussian-kernel-geodesic-Riemannian-manifold}, the kernel $K_\gamma:\Sym^{++}(n) \times \Sym^{++}(n) \mapto \R$ defined by
\begin{align}
	K_\gamma(A,B) = \exp(-\gamma d^2_{\bw}(A,B)), \;\; \gamma > 0,
\end{align}
{\it cannot} be positive definite $\forall \gamma > 0$ (but may be positive definite for some $\gamma > 0$).

{\bf Log-Euclidean metric}. We consider next the Log-Euclidean metric, which is widely used in many applications.
The Log-Euclidean Riemannian metric on $\Sym^{++}(n)$ was formulated by the authors of
\cite{Arsigny:2007LogE}.
It is the Riemannian metric arising from the following commutative Lie group multiplication on $\Sym^{++}(n)$,
\begin{align}
	\odot: \Sym^{++}(n) \times \Sym^{++}(n) &\mapto \Sym^{++}(n),
	\nonumber
	\\
	A \odot B &= \exp(\log(A) + \log(B)),
\end{align}
where $\log$ denotes the principal matrix logarithm.
The Log-Euclidean metric is a {\it bi-invariant} Riemannian metric on $(\Sym^{++}(n), \odot)$. 
Let us fix the inner product on $T_I(\Sym^{++}(n)) \cong \Sym(n)$ to be the Frobenius inner product. 
Then the Log-Euclidean metric is given by,
$\forall P \in \Sym^{++}(n)$,
\begin{align}
	\label{equation:metric-logE}
	g_{\logE}(P)(U,V) = \la D\log(P)(U), D\log(P)(V)\ra_F, \;\; U,V \in \Sym(n),
\end{align}
where $D\log(P)$ denotes the Fr\'echet derivative of the principal logarithm $\log$ at $P$. $(\Sym^{++}(n), g_{\logE})$
is a Riemannian manifold with {\it zero} sectional curvature. There is a unique geodesic joining any pair
$A,B \in \Sym^{++}(n)$, with closed form expression
\begin{align}
	\label{equation:geodesic-logE}
	\gamma_{\logE}^{AB}(t) = \exp((1-t)\log(A) + t\log(B)).
\end{align}
The Riemannian distance between $A$ and $B$ is the length of this geodesic: 
\begin{align}
	\label{equation:distance-logE}
	d_{\logE}(A,B) = ||\log(A) - \log(B)||_F.
\end{align}

{\bf Vector space structure of $\Sym^{++}(n)$}. 
While $\Sym^{++}(n)$ is  not a vector subspace of the Euclidean space  $(\Sym(n), +, \cdot, \la, \ra_F)$, it admits the following special vector space structure.
Together with the abelian group operation $\odot$, we define the following scalar multiplication 
on $\Sym^{++}(n)$ \cite{Arsigny:2007LogE}:
\begin{align}
	\myast: \R \times \Sym^{++}(n) &\mapto \Sym^{++}(n),
	\nonumber
	\\
	\lambda \myast A &= \exp(\lambda \log(A)) = A^{\lambda}, \;\;\; \lambda \in \R.
\end{align}
Endowed with the abelian group operation $\odot$ and the scalar multiplication $\myast$,
the  vector space  axioms  can be readily verified to show that
$(\Sym^{++}(n), \odot, \myast)$ is a vector space \cite{Arsigny:2007LogE}.

{\bf Inner product space structure on $\Sym^{++}(n)$}.
On top of the vector space structure $(\Sym^{++}(n), \odot, \myast)$, 
we define the following
{\it Log-Euclidean inner product}:
\begin{align}
	\label{equation:logE-inner}
	\la A, B\ra_{\logE} = \la \log(A), \log(B)\ra_F = \trace[\log(A)\log(B)].
\end{align}
along with the corresponding {\it Log-Euclidean norm}
\begin{align}
	||A||_{\logE}^2 = \la \log(A), \log(A)\ra_{F} = \trace[\log^2(A)].
\end{align}
The axioms of inner product, namely symmetry, positivity, and linearity with respect to
the operations $(\odot, \myast)$
can be readily verified
to show that 
\begin{align}
	(\Sym^{++}(n), \odot, \myast, \la \; , \; \ra_{\logE})
\end{align}
is an inner product space,
as first discussed in \cite{Li:2013LogE}. Furthermore, the following map is an isometrical isomorphism of inner product spaces:
\begin{align}
	\log: (\Sym^{++}(n), \odot, \myast, \la \; , \; \ra_{\logE}) \mapto (\Sym(n), +, \cdot, \la \;,\; \ra_F), \; A \mapto \log(A).
\end{align}
In terms of the Log-Euclidean inner product and Log-Euclidean norm, the Log-Euclidean distance $d_{\logE}$ in Eq.\eqref{equation:distance-logE}
is expressed as
\begin{align}
	d_{\logE}(A,B) = ||\log(A) - \log(B)||_F &=||A\odot B^{-1}||_{\logE}.
\end{align}

The inner product space structure of $(\Sym^{++}(n), \odot, \myast, \la \; , \; \ra_{\logE})$ allows us to define positive definite kernels on $\Sym^{++}(n)$ using the inner product $\la\; , \; \ra_{\logE}$ and the norm $||\;||_{\logE}$. We have the following result.

\begin{theorem} The following kernels $K: \Sym^{++}(n) \times \Sym^{++}(n) \mapto \R$ are positive definite:
	\label{theorem:kernel-PD}
	\begin{align}
		K(A,B) &= (\la A,B\ra_{\logE} + c)^d = (\la \log(A), \log(B)\ra_F + c)^d, \;\;\; c \geq 0, d \in \N.
		\label{equation:kernel-logE-inner}
		\\
		K(A,B) &= \exp\left(-\frac{1}{\sigma^2} ||A\odot B^{-1}||_{\logE}^p) \right)
		\nonumber
		\\
		& = \exp\left(-\frac{1}{\sigma^2} ||\log(A) - \log(B)||^p_F\right), \;\;\sigma \neq 0, 0 < p \leq 2.
		\label{equation:Gaussian-generalized}
	\end{align}
\end{theorem}

\subsection{Kernels defined with Bregman divergences}
\label{section:kernel-Bregman}

In the previous section, we discussed  $\Sym^{++}(n)$  from the viewpoint of  a Riemannian manifold, along with the corresponding induced geodesic distance.
In this section, we consider the Bregman divergences, which are distance-like functions arising from the open convex cone structure of $\Sym^{++}(n)$.
%
In particular, we discuss the Alpha Log-Determinant (Log-Det) divergences, which are obtained based on the strictly convex function $\phi(X) = -\log\det(X)$, $X \in \Sym^{++}(n)$.

Let us first recall the concept of Bregman divergence \cite{Bregman:1967}. Let $\Omega \subset \R^n$ be a convex set and $\phi:\Omega \mapto \R$ be a differentiable and strictly convex function. 
Then it defines the following {\it divergence} function on $\Omega$:
\begin{align}
	B_{\phi}(x,y) = \phi(x) - \phi(y) - \la \grad\phi(y), x- y\ra.
\end{align}
For example, with $\Omega = \R^n$ and $\phi(x) = ||x||^2$, we obtain the squared Euclidean distance $B_{\phi}(x,y) = ||x-y||^2$.
More generally, $\phi$ defines following family
of divergences \cite{Zhang:Divergence2004}, parametrized by a parameter $\alpha \in \R$,
\begin{align}
	\label{equation:d-alpha-phi}
	d^{\alpha}_{\phi}(x,y) = \frac{4}{1-\alpha^2}\left[\frac{1-\alpha}{2}\phi(x) + \frac{1+\alpha}{2}\phi(y) - \phi\left(\frac{1-\alpha}{2}x + \frac{1+\alpha}{2}y\right)\right],
\end{align}
with $d^{\pm 1}_{\phi}$ defined as the limits of $d^{\alpha}_{\phi}$ as $\alpha \approach \pm 1$. In fact,
we have
\begin{align}
	d^{1}_{\phi}(x,y) &= \lim_{\alpha \approach 1}d^{\alpha}_{\phi}(x,y) = B_{\phi}(x,y),
	\\
	d^{-1}_{\phi}(x,y) &= \lim_{\alpha \approach -1}d^{\alpha}_{\phi}(x,y) = B_{\phi}(y,x),
\end{align}
In general, it can be readily verified that $d^{\alpha}_{\phi}$ can be expressed 
in terms of the Bregman divergence $B_{\phi}$ $\forall \alpha \in \R$, as follows:
\begin{align}
	d^{\alpha}_{\phi}(x,y) = \frac{4}{1-\alpha^2}
	\left[\frac{1-\alpha}{2}B_{\phi}\left(x, \frac{1-\alpha}{2}x + \frac{1+\alpha}{2}y\right) + \frac{1+\alpha}{2}B_{\phi}\left(y, \frac{1-\alpha}{2}x + \frac{1+\alpha}{2}y\right)\right].
\end{align}



{\bf Alpha Log-Det divergences}. Consider $\Omega = \Sym^{++}(n)$  together with the function $\phi(X) = - \log\det(X)$, $X \in \Sym^{++}(n)$.
Fan's inequality \cite{kyfan:1950} on the log-concavity of the matrix determinant function on $\Sym^{++}(n)$ states that
\begin{align}
	\label{equation:KyFan}
	\det[\alpha A + (1-\alpha)B] \geq \det(A)^{\alpha}\det(B)^{1-\alpha},
	\;\forall A,B \in \Sym^{++}(n), \; 0 \leq \alpha \leq 1.
\end{align}
For $0 < \alpha < 1$, equality occurs if and only if $A = B$. 
Thus the function
$\phi(X) = -\log\det(X)$ is strictly convex on $\Sym^{++}(n)$. Hence, based on Eq.~(\ref{equation:d-alpha-phi}), we obtain 
the parametrized family of Alpha Log-Det divergences, 
as defined in \cite{Chebbi:2012Means}
\begin{align}
	\label{equation:d-alpha-logdet}
	d^{\alpha}_{\logdet}(A,B) = \frac{4}{1-\alpha^2}\log\left[\frac{\det(\frac{1-\alpha}{2}A + \frac{1+\alpha}{2}B)}{\det(A)^{\frac{1-\alpha}{2}}\det(B)^{\frac{1+\alpha}{2}}}\right], \;\; -1 < \alpha < 1,
\end{align}
with the limiting cases $\alpha = \pm 1$, obtained via L'Hopital's rule, given by
\begin{align}
	d^{1}_{\logdet}(A,B) &= \lim_{\alpha \approach 1}d^{\alpha}_{\logdet}(A,B) = \trace(B^{-1}A - I) - \log\det(B^{-1}A),
	\label{equation:d-alpha-logdet-1}
	\\
	d^{-1}_{\logdet}(A,B) &= \lim_{\alpha \approach -1}d^{\alpha}_{\logdet}(A,B) = \trace(A^{-1}B-I) - \log\det(A^{-1}B).
	\label{equation:d-alpha-logdet+1}
\end{align}
The following properties are immediate from Fan's inequality
\begin{align}
	d^{\alpha}_{\logdet}(A,B) &\geq 0,
	\\
	d^{\alpha}_{\logdet}(A,B) &= 0 \equivalent A = B.
\end{align}
Instead of symmetry, $d^{\alpha}_{\logdet}$ satisfies the {\it dual symmetry} property 
\begin{align}
	d^{\alpha}_{\logdet}(A,B) = d^{-\alpha}_{\logdet}(B,A).
\end{align}
In particular, $d^{\alpha}_{\logdet}$ is symmetric if and only if $\alpha = 0$, that is,
\begin{align}
	d^{0}_{\logdet}(A,B) = d^{0}_{\logdet}(B,A).
\end{align}

Having defined the Alpha Log-Det divergences $d^{\alpha}_{\logdet}$ on $\Sym^{++}(n)$,
as in the previous section, a natural question that arises is whether
Gaussian-like kernels can be defined using  $d^{\alpha}_{\logdet}$. By the symmetry property of kernels, the only case that could be considered is $\alpha  = 0$.
We have the following result from \cite{Sra:NIPS2012,Sra:2016Positive}.
\begin{theorem}
	[\cite{Sra:NIPS2012,Sra:2016Positive}] 
	\label{theorem:kernel-stein}
	The kernel $K:\Sym^{++}(n) \times \Sym^{++}(n) \mapto \R$, defined by
	\begin{align}
		K(A,B) = \exp\left(-\frac{\sigma}{4} d^0_{\logdet}(A,B)\right),
	\end{align}
	is positive definite if and only if $\sigma$ satisfies
	\begin{align}
		\label{equation:sigma-range}
		\sigma \in \left\{\frac{1}{2}, 1, \ldots, \frac{n-1}{2}\right\} \cup \left\{\sigma \in \R, \sigma > \frac{n-1}{2}\right\}.
	\end{align}
\end{theorem}
\begin{proof} 
	This theorem is a special case of Theorem VII.3.1 in \cite{Faraut1994} in the 
	general setting of Euclidean Jordan algebras and symmetric cones, with $\Sym(n)$ being the Euclidean Jordan algebra under the Jordan product $ A \circ B = \frac{1}{2}(AB + BA)$ and $\Sym^{++}(n)$ being the associated symmetric cone.
	
	The following elementary proof for the {\it if } part is based on that given in \cite{Sra:NIPS2012,Sra:2016Positive}.
	Assume that $\sigma$ satisfies \eqref{equation:sigma-range}. By definition of $d^{0}_{\logdet}$,
	\begin{align*}
		K(A,B) = \frac{\det(A)^{\frac{\sigma}{2}}\det(B)^{\frac{\sigma}{2}}}{\det(\frac{A+B}{2})^{\sigma}}.
	\end{align*}
	It thus suffices to show that the kernel function $H_{\sigma}: \Sym^{++}(n) \times \Sym^{++}(n) \mapto \R$ defined by $H_{\sigma}(X_1,X_2) = \det(X_1 + X_2)^{-\sigma}$ is positive definite with $\sigma$ as given in \eqref{equation:sigma-range}. For $X \in  \Sym^{++}(n)$, we have the Gaussian integral
	\begin{align*}
		\int_{\R^n}e^{-y^TXy}dy = \pi^{n/2}\det(X)^{-1/2}.
	\end{align*}
	Define the feature map $\varphi: \Sym^{++}(n) \mapto L^2(\R^n)$ by $\varphi(X)(y) = \frac{1}{\pi^{n/4}}e^{-y^TXy}$. Then
	for any pair $X_1,X_2 \in \Sym^{++}(n)$,
	\begin{align*}
		\la \varphi(X_1), \varphi(X_2)\ra_{L^2(\R^n)} = \frac{1}{\pi^{n/2}}\int_{\R^n}e^{-y^T(X_1+X_2)y}dy = \det(X_1 + X_2)^{-1/2}.
	\end{align*}
	It follows that $H_{1/2}$ is positive definite. Consequently, $H_{\sigma}$ is positive definite whenever $\sigma = \frac{k}{2}$ $\forall k \in \N$.
	
	For $\sigma > \frac{n-1}{2}$, $\sigma \in \R$, consider the matrix-variate Gamma function (see e.g. \cite{Mathai2022}, Chapter 5, Eq.(5.1.2))
	\begin{align}
		\Gamma_n(\sigma) =\int_{\Sym^{++}(n)}
		e^{-\trace(X)}\det(X)^{\sigma - \frac{n+1}{2}}dX, \; \;
		\sigma > \frac{n-1}{2},
	\end{align}
	along with the following identity (\cite{Mathai2022}, Eq.(5.2.2)), where $\forall B \in \Sym^{++}(n)$,
	\begin{align}
		\det(B) ^{-\sigma}=\frac{1}{\Gamma_n(\sigma)}\int_{\Sym^{++}(n)}
		e^{-\trace(BX)}\det(X)^{\sigma - \frac{n+1}{2}}dX, \; \;
		\sigma > \frac{n-1}{2}.
	\end{align}
	Define the feature map $\psi: \Sym^{++}(n) \mapto L^2(\Sym^{++}(n), \nu)$, where $d\nu(X) =  \frac{1}{\Gamma_n(\sigma)}\det(X)^{\sigma - \frac{n+1}{2}}dX$, by $\psi(X)(Y) = e^{-\trace(XY)}$. Then for any pair $X_1,X_2 \in \Sym^{++}(n)$,
	\begin{align*}
		\la \psi(X_1), \psi(X_2)\ra_{L^2(\Sym^{++}(n),\nu)} = \det(X_1 + X_2)^{-\sigma}.
	\end{align*}
	It thus follows that $H_{\sigma}$ is positive definite $\forall \sigma > \frac{n-1}{2}$.
\end{proof}

\subsection{Kernels defined with the Log-Hilbert-Schmidt metric}
\label{section:kernel-log-HS} 

In this section, we describe the  generalization of the Log-Euclidean metric
on $\Sym^{++}(n)$ and its corresponding kernels, as described in Section 
\ref{section:kernel-logE},
to the infinite-dimensional setting of positive definite Hilbert-Schmidt operators on a Hilbert space.
This generalization was first carried out in 
\cite{Minh:NIPS2014}.

Throughout the following, let $(\Hcal, \la, \ra)$ be a real, separable Hilbert space, with $\dim(\Hcal) = \infty$ unless explicitly stated otherwise.
For two separable Hilbert spaces $\Hcal_1, \Hh_2$, let $\Lcal(\Hcal_1,\Hcal_2)$ denote the Banach space of bounded linear operators from $\Hcal_1$ to $\Hcal_2$, with operator norm $||A||=\sup_{||x||_1\leq 1}||Ax||_2$.
For $\Hcal_1=\Hcal_2 = \Hcal$, we use the notation $\Lcal(\Hcal)$.
Let $\Sym(\Hcal) \subset \Lcal(\Hcal)$ be the set of bounded, self-adjoint linear operators on $\Hcal$. Let $\Sym^{+}(\Hcal) \subset \Sym(\Hcal)$ be the set of
self-adjoint, {\it positive} operators on $\Hcal$, i.e. $A \in \Sym^{+}(\Hcal) \equivalent A^{*}=A, \la Ax,x\ra \geq 0 \forall x \in \Hcal$. 
Let $\Sym^{++}(\Hcal)\subset \Sym^{+}(\Hcal)$ be the set of self-adjoint, {\it strictly positive} operators on $\Hcal$,
i.e $A \in \Sym^{++}(\Hcal) \equivalent A^{*}=A, \la x, Ax\ra > 0$ $\forall x\in \Hcal, x \neq 0$.
We write $A \geq 0$ for $A \in \Sym^{+}(\Hcal)$ and $A > 0$ for $A \in \Sym^{++}(\Hcal)$.
If $\gamma I+A > 0$, where $I$ is the identity operator,$\gamma \in \R,\gamma > 0$, then $\gamma I+A$ is also invertible, in which case it is called
{\it positive definite}. \color{black} In general, $A \in  \Sym(\Hcal)$ is said to be positive definite if $\exists M_A > 0$ such that $\la x, Ax\ra \geq M_A||x||^2$ $\forall x \in \Hcal$ - this condition is equivalent to $A$ being both strictly positive and invertible, see, e.g., \cite{Petryshyn:1962}.\color{black}
The Hilbert space $\HS(\Hcal_1,\Hcal_2)$ of Hilbert-Schmidt operators from $\Hcal_1$ to $\Hcal_2$ is defined by 
(see, e.g., \cite{Kadison:1983})
$\HS(\Hcal_1, \Hcal_2) = \{A \in \Lcal(\Hcal_1, \Hcal_2):||A||^2_{\HS} = \trace(A^{*}A) =\sum_{k=1}^{\infty}||Ae_k||_2^2 < \infty\}$,
for any orthonormal basis $\{e_k\}_{k \in \Nbb}$ in $\Hcal_1$,
with inner product $\la A,B\ra_{\HS}=\trace(A^{*}B)$. For $\Hcal_1 = \Hcal_2 = \Hcal$, we write $\HS(\Hcal)$. 

We now seek to generalize the expression $||\log(A) - \log(B)||_F$, $A,B \in \Sym^{++}(n)$, to the setting where $A,B$ are self-adjoint positive Hilbert-Schmidt operators on a separable Hilbert space. We first note the following {\it two crucial differences between the finite and infinite-dimensional settings}.

(i) Assume that $1\leq \dim(\Hcal) \leq \infty$. Assume that $A \in \Sym(\Hcal)$ is compact and strictly positive.
Then $A$ has a countable spectrum of positive eigenvalues $\{\lambda_k(A)\}_{k=1}^{\dim(\Hcal)}$, 
counting multiplicities, with $\lim\limits_{k \approach \infty}\lambda_k(A) = 0$ if $\dim(\Hcal) = \infty$.
If $\{\phi_k(A)\}_{k=1}^{\dim(\Hcal)}$ denote the corresponding normalized eigenvectors, then 
$A$ admits the spectral decomposition
\begin{align}
	A = \sum_{k=1}^{\dim(\Hcal)}\lambda_k(A) \phi_k(A) \otimes \phi_k(A),
\end{align}
where $\phi_k(A) \otimes \phi_k(A): \Hcal \mapto \Hcal$ is 
defined by
$
(\phi_k(A) \otimes \phi_k(A))w = \la w, \phi_k(A)\ra \phi_k(A)$, $w \in \Hcal.
$
The principal logarithm of $A$ is then given by
\begin{align}\label{equation:logdef}
	\log(A) = \sum_{k=1}^{\dim(\Hcal)}\log(\lambda_k(A)) \phi_k(A) \otimes \phi_k(A).
\end{align}
Clearly, $\log(A)$ is bounded if and only if $\dim(\Hcal) < \infty$, 
since for $\dim(\Hcal) = \infty$, we have $\lim\limits_{k\approach \infty}\log(\lambda_k(A)) = - \infty$.
Thus,
when $\dim(\Hcal) = \infty$, the condition that $A$ be strictly positive is {\it not} sufficient 
for $\log(A)$ to be bounded. This problem is resolved by considering the {\it regularized} or {\it unitized} operator
$A + \gamma I$, $\gamma \in \R$, $\gamma > 0$, which is {\it positive definite}, so that the following operator is always bounded:
\begin{align}\label{equation:logdef-unitized}
	\log(A+\gamma I) = \sum_{k=1}^{\infty}\log(\lambda_k(A) + \gamma) \phi_k(A) \otimes \phi_k(A).
\end{align}


(ii) The generalization of the Frobenius norm $||\;||_F$ to the Hilbert space setting is the Hilbert-Schmidt norm $||\;||_{\HS}$. 
Consider now the operators $A + \gamma I > 0$, $B + \mu I > 0$, with $A,B \in \HS(\Hcal)$. The expression
\begin{align}
	||\log(A+\gamma I) - \log(B + \mu I)||_{\HS}
\end{align} 
is generally infinite, however, since the identity operator $I$ is not Hilbert-Schmidt when $\dim(\Hcal) = \infty$, with 
$||I||_{\HS} = \infty$. Thus, for $A=B = 0$,  with $\gamma \neq \mu > 0$, we have $||\log(\gamma I)  -\log(\mu I)||_{\HS} = |\log(\gamma) - \log(\mu)|\;||I||_{\HS} = \infty$.
This problem is fully resolved by enlarging the set of Hilbert-Schmidt operators to include the identity operator, via the {\it extended Hilbert-Schmidt norm} and {\it inner product}, as follows.

{\bf Extended Hilbert-Schmidt operators}. 
In \cite{Larotonda:2007}, the author considered the following set of {\it extended}, or {\it unitized}, Hilbert-Schmidt operators
\begin{align}
	\HS_X(\Hcal) = \{A + \gamma I: A \in \HS(\Hcal), \gamma \in \R\}.
\end{align}
This set is a Hilbert space under the {\it extended Hilbert-Schmidt inner product}, under which the Hilbert-Schmidt and scalar operators are orthogonal, 
\begin{align}
	\la A+\gamma I, B + \mu I\ra_{\HS_X} = \la A,B\ra_{\HS} + \gamma\mu.
\end{align}
The corresponding {\it extended Hilbert-Schmidt norm} is given by
\begin{align}
	||A+\gamma I||^2_{\HS_X} = ||A||^2_{\HS} + \gamma^2.
\end{align}
In particular, $||I||_{\HS_X} = 1$, in contrast to $||I||_{\HS} = \infty$.

{\bf The Hilbert manifold of positive definite Hilbert-Schmidt operators}. 
Consider the following subset of {\it (unitized) positive definite Hilbert-Schmidt operators}

\newcommand{\PC}{\mathit{PC}}
\begin{align}
	\PC_2(\Hcal) = \{A+\gamma I > 0: A \in \Sym(\Hcal) \cap\HS(\Hcal), \gamma \in \R\} \subset \HS_X(\Hcal).
\end{align}

The set $\PC_2(\Hcal)$ is an open subset of the  Hilbert space
\begin{align}
	\Sym(\Hcal) \cap \HS_X(\Hcal) = \{A + \gamma I: A = A^{*},  A \in \HS(\Hcal), \gamma \in \R\}.
\end{align}
Thus  $\PC_2(\Hcal)$ is a Hilbert manifold modeled on $\Sym(\Hcal) \cap \HS_X(\Hcal)$.
If $(A+\gamma I) \in \PC_2(\Hcal)$, then it has a countable spectrum $\{\lambda_k(A) +\gamma\}_{k=1}^{\infty}$ satisfying $\lambda_k + \gamma \geq  M_A$ for some constant $M_A >0$,
and
$\log(A+\gamma I)$ as defined by (\ref{equation:logdef}) is well-defined and bounded,
with $\log(A + \gamma I) \in \Sym(\Hcal) \cap \HS_X(\Hcal)$.

The operations $(\odot, \myast)$ on $\Sym^{++}(n)$
as defined in 
Section \ref{section:kernel-logE}, 
can be readily generalized to $\PC_2(\Hcal)$ as follows 
\cite{Minh:NIPS2014}.
First, the commutative Lie group multiplication operation $\odot$ on $\PC_2(\Hcal)$ is defined by
\begin{align}
	\odot: \PC_2(\Hcal)  \times \PC_2(\Hcal) &\mapto \PC_2(\Hcal)
	\nonumber
	\\
	(A+\gamma I) \odot (B+\mu I) &= \exp(\log(A+\gamma I) + \log(B + \mu I)).
\end{align}
The Log-Hilbert-Schmidt metric is then a bi-invariant metric on $(\PC_2(\Hcal), \odot)$.
Let us choose the inner product on $T_I(\PC_2(\Hcal)) \cong \Sym(\Hcal) \cap\HS_X(\Hcal)$ to be
the extended Hilbert-Schmidt inner product. Then we have the following generalization of the Log-Euclidean metric in Section \ref{section:kernel-logE}
\begin{theorem}
	\label{theorem:logHS-metric}
	The Log-Hilbert-Schmidt metric is given by, $\forall P \in \PC_2(\Hcal)$,
	$\forall U, V \in \Sym(\Hcal)\cap \HS_X(\Hcal)$,
	\begin{align}
		\label{equation:metric-logHS}
		g_{\logHS}(P)(U,V) = \la D\log(P)(U), D\log(P)(V)\ra_{\HS_X},
	\end{align}
	where $D\log(P)$ denotes the Fr\'echet derivative of the principal logarithm $\log$ at $P$. $(\PC_2, g_{\logHS})$
	is an infinite-dimensional Riemannian manifold with {\it zero} sectional curvature. There is a unique geodesic joining any pair
	$(A+\gamma I), (B+\mu I) \in \PC_2(\Hcal)$, with closed form expression
	\begin{align}
		\label{equation:geodesic-logHS}
		\gamma_{\logHS}^{(A+\gamma I), (B+\mu I)}(t) = \exp((1-t)\log(A+\gamma I) + t\log(B+\mu I)).
	\end{align}
	The induced Riemannian distance between $(A+\gamma I)$ and $(B+\mu I)$ is the length of this geodesic: 
	\begin{align}
		\label{equation:distance-logHS}
		d_{\logHS}[(A+\gamma I),(B+\mu I)] = ||\log(A+\gamma I) - \log(B+\mu I)||_{\HS_X}.
	\end{align}
\end{theorem}

For any pair of operators $(A+\gamma I), (B+ \mu I) \in \PC_2(\Hcal)$, the distance $d_{\logHS}[(A+\gamma I), (B+ \mu I)] = ||\log(A+\gamma I) - \log(B+ \mu I)||_{\eHS}$
is always finite. Furthermore, when $\dim(\Hcal) = \infty$, by the orthogonality of the scalar and Hilbert-Schmidt operators, we have the decomposition
\begin{align}
	\label{equation:distance-logHS-decomp}
	||\log(A+\gamma I) - \log(B+ \mu I)||_{\eHS}^2 = \left\|\log\left(\frac{A}{\gamma} + I\right) - \log\left(\frac{B}{\mu} + I\right)\right\|^2_{\HS}  + \left(\log\frac{\gamma}{\mu}\right)^2.
\end{align}
In particular, for $A= B = 0$, Eq.~(\ref{equation:distance-logHS-decomp}) gives
\begin{align}
	d_{\logHS}[\gamma I, \mu I] = ||\log(\gamma I) - \log(\mu I)||_{\eHS} = |\log(\gamma/\mu)|.
\end{align}
Thus the second term on the right hand side of Eq.~(\ref{equation:distance-logHS-decomp}) is 
precisely the squared Log-Hilbert-Schmidt distance between the scalar operators $\gamma I$ and $\mu I$
(this distance would be infinite if measured in the Hilbert-Schmidt norm $||\;||_{\HS}$). 

{\bf Vector space structure of $\PC_2(\Hcal)$}. 
Together with the group operation $\odot$, we define the following scalar multiplication 
on $\PC_2(\Hcal)$
\begin{align}
	\myast: \R \times \PC_2(\Hcal) &\mapto \PC_2(\Hcal),
	\nonumber
	\\
	\lambda \myast (A+\gamma I) &= \exp(\lambda \log(A+\gamma I)) = (A+\gamma I)^{\lambda}, \;\;\; \lambda \in \R.
\end{align}
Endowed with the commutative group multiplication $\odot$ and the scalar multiplication $\myast$,
the  vector space axioms  can be readily verified to show that
$(\PC_2(\Hcal), \odot, \myast)$ is a vector space.

{\bf Hilbert space structure on $\PC_2(\Hcal)$}.
On top of the vector space structure $(\PC_2(\Hcal), \odot, \myast)$, 
we define the following
{\it Log-Hilbert-Schmidt inner product} on $(\PC_2(\Hcal), \odot, \myast)$ by
\begin{align}
	\label{equation:logHS-inner}
	\la (A+\gamma I), (B+ \mu I)\ra_{\logHS} = \la \log(A+\gamma I), \log(B+\mu I)\ra_{\eHS},
\end{align}
along with the corresponding {\it Log-Hilbert-Schmidt norm}
\begin{align}
	||A+\gamma I||_{\logHS}^2 = \la \log(A+\gamma I), \log(A+\gamma I)\ra_{\eHS}.
\end{align}
The axioms of inner product, namely symmetry, positivity, and linearity with respect to
the operations $(\odot, \myast)$
can be verified (see \cite{Minh:NIPS2014})
to show that 
\begin{align}
	(\PC_2(\Hcal), \odot, \myast, \la \; , \; \ra_{\logHS})
\end{align}
is a complete inner product space, that is, a Hilbert space.
This Hilbert space structure was first discussed in \cite{Minh:NIPS2014} and generalizes the 
finite-dimensional inner product space
\begin{align*}
	(\Sym^{++}(n), \odot, \myast, \la \; , \;\ra_{\logE})
\end{align*}
in Section \ref{section:kernel-logE}.
Also generalizing from  Section \ref{section:kernel-logE}, 
the map
\begin{align}
	\label{equation:iso-map-infinite}
	\log: (\PC_2(\Hcal), \odot, \myast, \la \; ,\;\ra_{\logHS}) \mapto (\Sym(\Hcal)\cap \HS_X(\Hcal), +, \cdot, \la \;,\;\ra_{\eHS})
\end{align}
is an isometrical isomorphism of Hilbert spaces, where the operations
$(+, \cdot)$ are the standard addition and scalar multiplication operations, 
respectively.

In terms of the Log-Hilbert-Schmidt inner product and Log-Hilbert-Schmidt norm, the Log-Hilbert-Schmidt distance $d_{\logHS}$ in Eq.~(\ref{equation:distance-logHS}) is expressed as
\begin{align}
	d_{\logHS}(A+\gamma I),(B + \mu I)] &= ||\log(A+\gamma I) - \log(B + \mu I)||_{\eHS}
	\nonumber
	\\
	&=||(A+\gamma I)\odot (B+\mu I)^{-1}||_{\logHS}
	\\
	&= \sqrt{\la (A+\gamma I)\odot (B+\mu I)^{-1}, (A+\gamma I)\odot (B+\mu I)^{-1}\ra_{\logHS}}.
	\nonumber
\end{align}

The Hilbert space structure of $(\PC_2(\Hcal), \odot, \myast, \la \; , \; \ra_{\logHS})$ allows us
to define positive definite kernels on $\PC_2(\Hcal)$ using the inner product 
$\la \; , \;\ra_{\logHS}$ and the norm $||\;||_{\logHS}$. The following are 
the infinite-dimensional generalizations of the Log-Euclidean kernels defined in Section \ref{section:kernel-logE}, see
\cite{Minh:NIPS2014}:

\begin{theorem} 
	\label{theorem:kernel-PD-logHS}
	The following kernels $K: \PC_2(\Hcal) \times \PC_2(\Hcal) \mapto \R$ are positive definite:
	\begin{align}
		K[(A+\gamma I), (B+\mu I)] &= (\la (A+\gamma I), (B+\mu I)\ra_{\logHS} + c)^d, \;\; c \geq 0, \; d \in \N
		\nonumber
		\\
		&=(c + \la \log(A+\gamma I), \log(B+ \mu I)\ra_{\eHS})^d, 
		\label{equation:logHS-kernel-inner}
	\end{align}
	\begin{align}
		K[(A+\gamma I), (B+\mu I)] &= \exp\left(-\frac{||(A+\gamma I) \odot (B+ \mu I)^{-1}||_{\logHS}^p}{\sigma^2}\right) 
		\nonumber
		\\
		& = \exp\left(-\frac{||\log(A+\gamma I) - \log(B+ \mu I)||_{\eHS}^p}{\sigma^2}\right), 
		\label{equation:logHS-kernel-distance}
	\end{align}
	for $\sigma \neq 0$, $0 < p \leq 2$.
\end{theorem}

Theorem \ref{theorem:kernel-PD-logHS} thus generalizes Theorem \ref{theorem:kernel-PD}
to the infinite-dimensional setting. In particular, for $\Hcal = \R^n$, $\PC_2(\Hcal) = \Sym^{++}(n)$,
$\gamma = \mu = 0$, and $A,B \in \Sym^{++}(n)$, we recover Theorem \ref{theorem:kernel-PD}.

{\bf The RKHS setting}
We now present a concrete setting, namely that of RKHS covariance operators, which is of particular interest computationally and practically, since many quantities
admit closed forms via kernel Gram matrices which can be efficiently computed.
This setting has been applied in problems in machine learning and computer vision, see e.g.
\cite{ZhouChellapa:PAMI2006ProbDistance,
Minh:NIPS2014,Minh:Covariance2018}.

In the following, we assume that $\Xcal$ is a complete separable metric space,
{$K: \Xcal \times \Xcal \mapto \R$} is a continuous positive definite kernel,
$\rho$
is a Borel probability measure on $\Xcal$, such that
$\int_{\Xcal}K(x,x)d\rho(x) < \infty$.
Let $\Hcal(K)$ be the reproducing kernel Hilbert space (RKHS) induced by {$K$}, then $\Hcal(K)$ is separable (\cite{Steinwart:SVM2008}, Lemma 4.33).
Let {$\Phi: \Xcal \mapto \Hcal(K)$} be the corresponding canonical feature map
$\Phi(x) = K_x$, {where  $K_x:\Xcal \mapto \R$ is defined by } $K_x(y) = K(x,y) \forall y \in \Xcal$.
Then $K(x,y) = \la \Phi(x), \Phi(y)\ra_{\Hcal(K)}$ $\forall (x,y) \in \Xcal \times \Xcal$
and the probability measure $\rho$ satisfies 
\begin{align}
	\int_{\Xcal}||\Phi(x)||_{\Hcal(K)}^2d\rho(x) = \int_{\Xcal}K(x,x)d\rho(x) < \infty.
\end{align}
The following RKHS mean vector $\mu_{\Phi} \in \Hcal(K)$ and RKHS covariance operator {$C_{\Phi}:\Hcal(K) \mapto \Hcal(K)$} induced by the feature map $\Phi$ are then well-defined:
\begin{align}
	\mu_{\Phi} &= \mu_{\Phi,\rho} = \int_{\Xcal}\Phi(x)d\rho(x) \in \Hcal(K), \;\;\;
	\\
	C_{\Phi} &= C_{\Phi,\rho} = \int_{\Xcal}(\Phi(x)-\mu_{\Phi})\otimes (\Phi(x)-\mu_{\Phi})d\rho(x).
\end{align}
Here $C_{\Phi}$ is a positive, trace class operator on $\Hcal(K)$.
Let {$\Xbf =[x_1, \ldots, x_m]$,$m \in \Nbb$,} be a data matrix randomly sampled from {$\Xcal$} according to
$\rho$, where {$m \in \Nbb$} is the number of observations.
The feature map {$\Phi$} on {$\Xbf$} 
defines
the bounded linear operator
$\Phi(\Xbf): \R^m \mapto \Hcal(K), \Phi(\Xbf)\b = \sum_{j=1}^mb_j\Phi(x_j) , \b \in \R^m$.
The corresponding RKHS empirical mean vector and RKHS covariance operator for {$\Phi(\Xbf)$}
are defined to be
\begin{align}
	\mu_{\Phi(\Xbf)} &= \frac{1}{m}\sum_{j=1}^m\Phi(x_j) = \frac{1}{m}\Phi(\Xbf)\myone_m,
	\\
	C_{\Phi(\Xbf)} &= \frac{1}{m}\sum_{j=1}^m(\Phi(x_j) - \mu_{\Phi(\Xbf)})\otimes(\Phi(x_j) - \mu_{\Phi(\Xbf)})
	\\
	& = \frac{1}{m}\Phi(\Xbf)J_m\Phi(\Xbf)^{*}: \Hcal(K) \mapto \Hcal(K),
	\label{equation:covariance-operator}
\end{align}
where $J_m = I_m -\frac{1}{m}\myone_m\myone_m^T,\myone_m = (1, \ldots, 1)^T \in \R^m$, is the centering matrix.

For concreteness, consider the task of image classification in computer vision, see e.g.  \cite{ZhouChellapa:PAMI2006ProbDistance,
	Minh:NIPS2014,Minh:Covariance2018}.
In this setting, $\Xbf$ is a data matrix of features obtained from an image and $C_{\Phi(\Xbf)}$ is a representation of that image, the so-called {\it covariance operator representation}.
This representation gives rise to
powerful nonlinear algorithms with
substantial improvements over finite-dimensional covariance matrices, which are a special case of the RKHS formulation when $K$ is the linear kernel, i.e. $K(x,y) = \la x, y\ra$ on $\R^n \times \R^n$.
With each image being represented by an RKHS covariance operator, for the task of image classification it is necessary to have a measure of similarity/dissimilarity between them.

Let $\rho_1,\rho_2$ be two Borel probability measures on $\Xcal$ satisfying $\int_{\Xcal}K(x,x)d\rho_i(x) < \infty$, $i=1,2$.
Let $\Xbf^i = (x^i_j)_{j=1}^{m_i}$, $i=1,2$, be randomly sampled from $\Xcal$ according to $\rho_i$.
%
Let $\mu_{\Phi(\Xbf^1)}, \mu_{\Phi(\Xbf^2)}$ and $C_{\Phi(\Xbf^1)}$, $C_{\Phi(\Xbf^2)}$
be the corresponding mean vectors and covariance operators induced by
the kernel 
$K$, respectively.
Define the following
finite Gram matrices:
{
\begin{align}
	&K[\Xbf^1] = \Phi(\Xbf^1)^{*}\Phi(\Xbf^1) \in \R^{m_1\times m_1},\;K[\Xbf^2] = \Phi(\Xbf^2)^{*}\Phi(\Xbf^2)\in \R^{m_2\times m_2}, 
	\\
	&K[\Xbf^1,\Xbf^2] = \Phi(\Xbf^1)^{*}\Phi(\Xbf^2) \in \R^{m_1 \times m_2},
	\\
	&(K[\Xbf^1])_{jk} = K(x_j^1,x_k^1), (K[\Xbf^2])_{jk} = K(x_j^2,x_k^2),(K[\Xbf^1,\Xbf^2])_{jk} = K(x^1_j, x^2_k).
\end{align}
}
Let $\gamma_i >0$, $i=1,2$, be fixed. Then the Log-Hilbert-Schmidt distance between 
$(C_{\Phi(\Xbf^1)} + \gamma_1I_{\Hcal(K)})$ and $(C_{\Phi(\Xbf^2)}+\gamma_2I_{\Hcal(K)})$ admits a closed form formula via the Gram matrices, as follows.
\begin{align}
&||\log(C_{\Phi(\Xbf^1)} + \gamma_1I_{\Hcal(K)})-\log(C_{\Phi(\Xbf^2)}+\gamma_2I_{\Hcal(K)})||^2_{\HS_X}
\nonumber
\\
&= ||\log(I_{m_1}+A^{*}A)||^2_{F} + ||\log(I_{m_2}+B^{*}B)||^2_{F}
- 2\trace[B^{*}Ah(A^{*}A)A^{*}Bh(B^{*}B)]
\nonumber
\\
&\quad + \left(\log\frac{\gamma_2}{\gamma_1}\right)^2. 
\end{align}
Here $A^{*}A = \frac{1}{m_1\gamma_1}J_{m_1}K[\Xbf^1]J_{m_1}$, $B^{*}B = \frac{1}{m_2\gamma_2}J_{m_2}K[\Xbf^2]J_{m_2}$, 
\\
$A^{*}B = \frac{1}{\sqrt{m_1m_2\gamma_1\gamma_2}}J_{m_1}K[\Xbf^1,\Xbf^2]J_{m_2}$, $B^{*}A = \frac{1}{\sqrt{m_1m_2\gamma_1\gamma_2}}J_{m_2}K[\Xbf^2,\Xbf^1]J_{m_1}$,
and $h(A) = A^{-1}\log(I+A)$, $A \in \Sym^{+}(\Hcal)$, with $h(0) = I$.
%

{\bf A two-layer kernel machine}. The Log-Hilbert-Schmidt inner product and distance between RKHS covariance operators gives rise to the following two-layer kernel machine, used in, e.g., image classification, as follows. In the first layer, a kernel $K_1$ is applied to the extracted image features, so
that each image is represented via an RKHS covariance operator. The Log-Hilbert-Schmidt distances/inner products between the RKHS covariance operators are then computed and another kernel $K_2$ is defined using these distances/inner products. A kernel algorithm, e.g. for classification, is then readily applied using the kernel $K_2$. 
All computations are carried out via the corresponding kernel Gram matrices. 
It has been demonstrated that the extra nonlinearity, via the addition of the first kernel layer, generally results in substantially better practical performances than kernel methods with the Log-Euclidean metric, where only the second kernel layer is present. 
We refer to \cite{Minh:NIPS2014,Minh:Covariance2018} for  further details of the actual practical experiments.

\section{Geometric manifold  learning techniques} \label{sec:manifoldlearning}

Data-driven sciences are widely regarded as the next paradigm that can fundamentally change sciences and pave the way for a new industrial revolution. In passing now  from (merely) topological to {\em geometric} data analysis we see now that differential, computational and discrete geometry have achieved first and great successes in data characterization and modelling. In particular, geometric deep learning has significantly advanced the capability of learning models for data with complicated topological and geometric structures.  The combination of geometric methods with learning models has thus great potential to fundamentally change the data sciences, and the involved disciplines, methods, and techniques
nowadays include, but are not confined to, discrete exterior calculus and Laplace Operators, discrete optimal transport, and geometric flow, discrete Ricci (like Olivier and Forman type) curvatures,  conformal geometry, combinatorial Hodge theory, dimension reduction via manifold learning, Isomap, Laplacian eigenmaps, diffusion maps, hyperbolic geometry, Poincaré embeddings, etc., geometric signal processing and deep learning, graph, simplex  and hypergraph neural networks, index theory, information geometry and (Gromov-) Hausdorff distance.

Given this huge and diverse amount of topics, it is clearly impossible to cover all of them here  in desirable detail.
In this section, we will therefore concentrate on just two, but paradigmatic ones, about 'learning' Riemannian manifolds, namely, the result of Belkin-Niyogi justifying mathematically, at least to a certain extent, 
the wide-spread use of Laplacian eigenmap techniques, and
the foundational theoretical work of Fefferman et al. on (Riemannian) manifold reconstruction.

\subsection{Some Generalities on Manifold Learning}
\label{section:manifold learning}
Broadly speaking, manifold learning stands for  a variety of techniques used in machine learning and data analysis to understand the underlying structure of high-dimensional data.  
 Originally,  ``manifold learning techniques"  refer to  non-linear  techniques  of dimension reduction, i.e., techniques of  transforming an initial representation of items in some high dimensional space
into a representation of these items in a space of lower dimension while preserving
	certain relevant properties of the initial representation.  We shall consider a more general concept of ``manifold learning" by interpreting ``relevant properties of the initial representation of items" as
	characteristics of the underlying geometric structure of a Riemannian (sub)manifold
	$M$ underlying the items and residing in a metric space $X$. Thus, in this broader sense, ``manifold learning"
signifies learning these characteristics from data points $x_i \in M \subset  X$, which have
	been distributed by the Riemannian volume form on $M$.

\subsection{Manifold Learning using spectral properties \`a la Belkin-Niyogi}
\label{section:Belkin-Niyogi}

We consider the typical scenario in which the given data reside on a low-dimensional manifold embedded in a 
higher-dimensional Euclidean space.  One popular approach for dealing with this situation is to construct a graph that
approximately represents the manifold and embed this graph into a low dimensional Euclidean space.
Examples of methods utilizing this approach include Isomap, Locally Linear Embedding (LLE), and Laplacian eigenmaps, among many others. As an example, we now discuss  a theorem by  Belkin and Niyogi  (Theorem \ref{theorem:BelkinNiyogi2008-Laplacian-convergence})    motivating their Laplacian eigenmap  algorithm \cite{Belkin-Niyogi2001}, which  we interpret  as  an instance  of Manifold Learning (Remark \ref{rem:BelkinNiyogi2008-Laplacian-convergence}). 

Let $(M, g)$ be an $n$-dimensional compact connected Riemannian  submanifold of Euclidean $d$-dimensional  space  $(\R^d, \la, \ra)$ for  some $d > n$. \footnote{By the famous Nash  embedding theorem, any Riemannian manifold  $(M, g)$  can be isometrically  embedded into a Euclidean space \cite{Nash1956}). However, this result also comes with its own  'curse of dimensionality' - namely, if $M$ has dimension $n$, then  $d$ is in general of order $n^3$. More precisely: if $M$ is compact, $d\le \frac{1}{2}n(3n+11)$, and otherwise, $d\le \frac{1}{2}n(3n+11)(n+1)$.
}

We shall use the notation $\la , \ra$ and $\|\cdot \|$ for  the Euclidean scalar product and  the Euclidean norm, respectively,   and also  for the  induced  metric $g$ and the associated  norm on  $TM\subset \R^n$, respectively. 
Given a set of data points $S_m = \{x_i\}_{i=1}^m$ in  $M$, following Belkin and Niyogi \cite{Belkin-Niyogi2001,BelkinNiyogi2003}, we shall   associate  to $S_m$     a  one-parameter   family of  weighted full  graphs $ S_m (t)$, where $t \in \R^+$,
whose  vertices  are elements  of $S_m$ and       whose  edge  $e_{ij}$ connecting   $x_i$ with $x_j$   has weight  
$$W^t_{ij}  =   e^{-\frac{||x_i-x_j||^2_{\R^d}}{4t}}. $$
Each of  these  graphs then  encodes    full information of the extrinsic  distance between each pair $(x_i, x_j)$ in $S_m$, when regarded as points  in $\R^d$.   Note that  we can compute  the extrinsic  distance  between  $x_i$ and $ x_j $ but  not  their   intrinsic  distance, i.e., the Riemannian   distance, since    the  Riemannian geometry of   $M$ is unknown. 
However, 
a theorem due to  Belkin and Niyogi \cite{Belkin-Niyogi2001,BelkinNiyogi2008}   states  that  we can  learn   the    spectral geometry  of $M$    from the   spectral  geometry of    graphs $S_m(t)$   in probability, as $m$ goes  to infinity.

To   each graph $S_m (t)$ we assign  a  symmetric positive semi-definite  bilinear  form ${\Delta^t_S}_{ m}$   on $\R^m$   by  setting for any $  z= (z_1, \ldots, z_m) \in \R^m$
	\begin{equation}
	\label{equation:Laplacian-scalar}
	{\Delta^t_S}_{m} (z, z): = \frac{1}{2}\sum_{i,j=1}^m W^t_{ij}(z_i - z_j)^2.
	\end{equation}
The bilinear form ${\Delta^t_S}_{m}$ is called the {\it unnormalized graph Laplacian} associated with the graph $S_m(t)$. 
We also identify $\Delta^t_{S_m}$ with  a self-adjoint  operator  on $\R^m$.
 For a function $f:M\to \R$, we have the vector $\mathbf{f} = (f(x_i))_{i=1}^m \in \R^m$ and   we  shall  regard $\Delta^t_{S_m}$ as an operator acting on $f$ at the data points $x_i$ as follows:
\begin{align*}
		(\Delta^t_{S_m} f)(x_i): = (\Delta^t_{S_m} \mathbf{f})(x_i) =  f(x_i) \sum_{j=1}^m e^{-\frac{||x_i - x_j||^2_{\R^d}}{4t}} - \sum_{j=1}^m f(x_j)e^{-\frac{||x_i - x_j||^2_{\R^d}}{4t}}.		
\end{align*}
		This motivates the definition of the following {\it point cloud Laplace operator} acting on any $f:M \to \R$, where $\forall x \in M$, one sets
		\begin{align*}
		(\mathbf{\Delta}^t_{S_m}f)(x) = f(x) \frac{1}{m}\sum_{j=1}^m e^{-\frac{||x - x_j||^2_{\R^d}}{4t}} - \frac{1}{m}\sum_{j=1}^m f(x_j)e^{-\frac{||x - x_j||^2_{\R^d}}{4t}}.
		\end{align*}
		Here $(\mathbf{\Delta}^t_{S_m}f)(x_i) = \frac{1}{m}(\Delta^t_{S_m} f)(x_i)$.
Denote by  $\vol_g$  the Riemannian volume form on $ (M, g)$.		
For a compact manifold $M$, its volume is finite and the Riemannian volume form gives rise to the uniform probability distribution  $\mu \in \Pp (M)$ via 
		$d\mu(p) = \frac{\vol_g(p)}{\vol_g(M)}$.
		The following result then shows that $\mathbf{\Delta}^t_{S_m}$ is an empirical version of $\Delta_g$.
\begin{theorem}[\cite{BelkinNiyogi2008}, Theorem 3.1]
\label{theorem:BelkinNiyogi2008-Laplacian-convergence}
Let $M$ be a compact $n$-dimensional Riemannian submanifold in $\R^d$.
Let $S_m = \{x_i\}_{i=1}^m$ be sampled according to the uniform probability distribution $\mu$ on $M$. 
Let $f \in C^{\infty}(M)$. Put $t_m = m^{-\frac{1}{n+2+\alpha}}$, where $\alpha > 0$. Then $\forall p \in M$, and for any fixed $\epsilon > 0$,
			\begin{align*}
			\lim_{m \rightarrow \infty}\mu^{m}\left\{S_m \in M^m: \left|\frac{1}{t_m (4\pi t_m)^{n/2}}\Deltabf^t_{S_m}f(p) - \frac{1}{\vol_g(M)}(\Delta_gf)(p)\right| > \epsilon \right\} = 0.
			\end{align*}
\end{theorem}  
\begin{remark}\label{rem:BelkinNiyogi2008-Laplacian-convergence}
\color{black}(1)  
Theorem \ref{theorem:BelkinNiyogi2008-Laplacian-convergence} can be interpreted as a statement  on  the  successful learning  of   the Laplace operator   of   a  Riemanian submanifold  $M\stackrel{i}{\INTO}  \R^N$      by   empirical   data   $S_m = \{x_i\}_{i=1}^m$ in a setting of  unsupervised learning.  Here  the  hypothesis  class  of  possible  approximators of  the  Riemannian  Laplace operator   consists  of point cloud    Laplace operators.  To measure   the deviation  of
the  point  cloud  Laplace  operators  from   the  Laplacian operator  $\Delta_g$ of   $M$  we use  the family of   loss  functions, parameterized  by $p \in M$ and  $f \in   C^\infty (M)$  defined  in Theorem \ref{theorem:BelkinNiyogi2008-Laplacian-convergence}.   

(2) 
The      investigation  of   the spectral geometry  of a Riemannian  manifold $(M, g)$ from its  discretizations  has been initiated  by  Dodziuk  \cite{Dodziuk1976} and Dodziuk-Patodi \cite{DP1976}.  Since   then  there  have been  many  papers in Riemannian  geometry  devoted  to this problem, see e.g.  Burago-Ivanov-Kurylev \cite{BIK2014} and the references  therein.  A    principal   difference  between Belkin-Niyogi's   and the aforementioned  results  in Riemannian  geometry  is that   the former considers  the  extrinsic  distance  between  data points  $x_i \in M \subset  \R^d$, which is  indeed very natural    from a  data  science    point of view. Another   principal  difference  between  Belkin-Niyogi's result  and the ones  just mentioned  is that  the convergence  in Theorem \ref{theorem:BelkinNiyogi2008-Laplacian-convergence} is  convergence  in probability,   whereas 
	the convergence in the others  takes place  under different    assumptions.  Notice, however, that spectral invariants are  not  complete  invariants  of   a  Riemannian manifold, i.e., isospectrality does not imply isometry,
	see,  e.g.,  Gordon-Webb-Wolpert \cite{GWW1992}.
	
(3) One may also  ask whether there are analogues or extensions of Belkin-Niyogi's result to other natural differential operators on Riemannian manifolds. For example, in the case where $M$ is spin, is there a corresponding theorem for the Dirac operator on $M$?

\end{remark}
\subsection{Riemannian Manifold Reconstruction \`a la Fefferman et al.}
\label{section:Fefferman manifold reconstruction}

Most results in manifold learning, as the one by Belkin-Niyogi discussed above, assume {\em a priori} the existence of a (Riemannian) manifold fitting the given data, though the manifold itself will actually essentially always remain unknown.
However, let us now consider a sort of converse problem, namely: Suppose that we are given a set of data and distances between its respective elements, so that we may think of it as a metric space $X$  (with, at least in data science applications, a usually finite, but nevertheless large number of elements). 
Will there then be an algorithm  to construct a  Riemannian manifold $(M,g)$  from $X$ so that the further study of $X$, might, in particular,  be amenable to techniques from global analysis and differential geometry?

Fundamental, and at least in our opinion, not merely mathematically sound, but indeed seminal work in this direction has recently been done by Fefferman et al. in a series of papers which we would now like to draw attention upon. Compare here, in particular, \cite{FIKNL2020} as well as \cite{FILN2020} and the further references therein.

Indeed, in  \cite{FIKNL2020}  the authors thoroughly investigate how a Riemannian manifold $(M,g)$ may best interpolate a given metric space $(X,d)$. For such an approximation,  a smooth $n$-dimensional manifold with Riemannian metric
has to be constructed, and determining such a Riemannian manifold does 
involve the construction of its topology, differentiable structure, and Riemannian metric. 
As their main result, the authors provide sufficient as well as necessary conditions to ensure that a metric space $X$ can be approximated, in the Gromov–Hausdorff or quasi-isometric sense, by a Riemannian manifold $(M,g)$ of a fixed dimension and with bounded diameter, sectional curvature, and injectivity radius, see  (\cite{FIKNL2020}, Theorem $1$). For this to hold,
$X$ should locally be metrically close to some Euclidean space and globally be endowed with an almost intrinsic metric,
and $M$ is constructed as a  submanifold of a separable Hilbert space, that is either $\R^d$ or $\ell^2$
(though the Riemannian metric $g$ is, in general, not equal to the induced submanifold one).

Moreover, in \cite{FILN2020} the authors take an important step further, in particular to challenges arising in machine learning, by considering the task of approximating a Riemannian manifold $(M,g)$ from the geodesic distances of points in a discrete subset of $M$, where the Riemannian manifold $(M,g)$ is considered as an abstract metric space with intrinsic distances (and not just as an embedded submanifold of an ambient Euclidean space), and, in addition, some 'noise' is present. 
Indeed, if $(M,g)$ is an (unknown) Riemannian manifold, $ {X_1, X_2, \ldots ,X_N }$ a set of $N$ sample points randomly sampled from $M$ and, in terms of the induced geodesic distance $d_M$  on $M$,
$D_{jk} := d_M (X_j , X_k ) + \eta_{jk}$ , where $j, k = 1, 2, \ldots, N$ and $\eta_{jk}$
are independent, identically distributed random variables (subject to certain natural conditions) are given, 
the authors prove that for $N$ getting larger and larger, one can construct a Riemannian manifold  $(M^*, g^*)$ 
approximating  $(M,g)$ w.r.t. the Lipschitz distance in higher and higher probability.

\

\section{Final remarks}\label{sec:final}

\

In this  paper  we demonstrated  the usefulness of the  languague of the Markov category  of probabilistic morphisms in statistical learning  theory. We showed  that the learnability of learning  algorithms  in supervised learning can be derived in particular from    geometric properties of the  graph operator of   probabilistic morphisms that formalize  regular  conditional   probability  operators in supervised learning theory. Our categorical and geometrical  framework is a natural  generalization  of  Vapnik-Stephanyuk's representation of regular  conditional probabilities  as a solution of multidimensional Fredholm integral equations \cite{Vapnik1998} \cite{Vapnik2000}, \cite{VI2016}.  For applications  of categorical methods to  Bayesian  statistics, we also refer to  \cite{JLT2021}, \cite{Le2025}.

We   also  considered  geometric  methods  for  extending  important kernel techniques  in Machine Learning  and showed that  important geometric  properties of Riemannian  manifolds  can be learned  from  point cloud data.

\section*{Acknowledgments}


The contributions  by HVL and WT were supported by the Research Institute for Mathematical Sciences,
an International Joint Usage/Research Center located in Kyoto University during  their   visit  to RIMS   in September 2023. All the authors would like to  thank  the  ananymous referee for several  helpful comments.







\

\noindent
Further funding:    The research of HVL was additionally supported by the Institute of Mathematics, Czech Academy of Sciences (RVO 67985840)  and  GA\v CR-project  GA22-00091S,  and the research of WT was additionally supported from the German-Japanese university consortium HeKKSaGOn.

\

\

\end{document}